\documentclass[11pt]{amsart}
\usepackage[T1]{fontenc}
\usepackage{graphicx}
\usepackage[utf8]{inputenc}
\usepackage{url}
\usepackage{hyperref}
\usepackage{bbm}
\usepackage[numbers,sort]{natbib}

\input{xy}
\xyoption{all}
\xyoption{poly}
\usepackage[all]{xy}
\usepackage{float}
\usepackage{enumerate}
\usepackage[usenames]{color}
\usepackage[mathscr]{eucal}

\usepackage{amsmath,amsthm,amsfonts,amssymb}
\usepackage{xcolor}

\usepackage[font=footnotesize,labelfont=bf]{caption}
 
      \theoremstyle{plain}
      \newtheorem{theorem}{Theorem}[section]
      \newtheorem*{theorem*}{Theorem}
      \newtheorem{lemma}[theorem]{Lemma}
      \newtheorem{corollary}[theorem]{Corollary}
      \newtheorem*{corollary*}{Corollary}
      \newtheorem{proposition}[theorem]{Proposition}
      \newtheorem*{proposition*}{Proposition}

      \theoremstyle{definition}
	  \newtheorem{example}[theorem]{Example}
      \newtheorem{definition}[theorem]{Definition}
      
     \theoremstyle{remark}
      \newtheorem{remark}[theorem]{Remark}
      
 \def\len{{\mathrm{len}}}
 \def\Filt{{\mathrm{Filt}}}

 \def\Cech{{\mathrm{\check{C}ech}}}
 
 \def\Rips{{\mathrm{Rips}}} 
 \def\dgm{{\mathrm{dgm}}}
 \def\inj{{\mathrm{inj}}}
 \def\Vol{{\mathrm{Vol}}}
 \def\diam{{\mathrm{diam}}}
 \def\convex{{\mathrm{conv}}}

 \def\M{{\mathcal M}}
 \def\N{{\mathcal N}}
 \def\O{{\mathcal O}}

 \def\ZZ{{\mathbb Z}}
 \def\RR{{\mathbb R}}
 \def\PP{{\mathbb P}}
 \def\XX{{\mathbb X}}
 \def\YY{{\mathbb Y}}
 
 \def\WW{{\mathbb W}}

\begin{document}

\title[Intrinsic persistent homology via density-based metric learning]{
Intrinsic persistent homology via density-based metric learning
}

\author[X. Fernández]{Ximena Fernández}
\address{Department of Mathematics, Swansea University, UK and Departamento de Matemática, FCEN, Universidad de Buenos Aires, Argentina.}
\email{x.fernand@dm.uba.ar}
\author[E. Borghini]{Eugenio Borghini}
\address{Departamento de Matemática and IMAS-CONICET, FCEN, Universidad de Buenos Aires, Argentina.}
\email{eborghini@dm.uba.ar}
\author[G. Mindlin]{Gabriel Mindlin}
\address{IFIBA, CONICET and Departamento de Física, FCEN, Universidad de Buenos Aires, Argentina}
\email{gabo@df.uba.ar}
\author[P. Groisman]{Pablo Groisman}
\address{Departamento de Matemática and IMAS-CONICET, FCEN, Universidad de Buenos Aires, Argentina and NYU-ECNU Institute of Mathematical Sciences at NYU Shanghai.}
\email{pgroisma@dm.uba.ar}

\begin{abstract}%
We address the problem of estimating topological features from data in high dimensional Euclidean spaces under the manifold assumption. Our approach is based on the  computation of persistent homology of the space of data points endowed with a sample metric known as Fermat distance.
We prove that such metric space converges almost surely to the manifold itself endowed with an intrinsic metric that accounts for both the geometry of the manifold and the density that produces the sample. This fact implies the convergence of the associated persistence diagrams. The use of this intrinsic distance when computing persistent homology presents advantageous properties such as robustness to the presence of outliers in the input data and less sensitiveness to the particular embedding of the underlying manifold in the ambient space. We use these ideas to propose and implement a method for pattern recognition and anomaly detection in time series, which is evaluated in applications to real data.
\end{abstract}



\maketitle

\section{Introduction}

\subsection{ Motivation and Problem Statement.}

It is a common situation in machine learning that the given data represents a  possibly noisy finite sample of a geometric object embedded in a high dimensional Euclidean space. 
This is the case, for instance, in the analysis of time series arising from observations of a dynamical system,
where a spatial representation of the data can be interpreted as a sample of a geometric structure --- the \textit{attractor} --- encoding valuable information of the underlying system's behaviour.
Under the manifold assumption, both the metric and the density of the sample play a central role in the process of reconstruction of topological properties of the underlying shape.

From a theoretical point of view, the problem can be stated as follows. Let $\mathbb X_n$  be a set of $n$ sample points with common density $f$ supported on a smooth compact Riemannian manifold $\M$ embedded in $\RR^D$. We are interested in recovering topological features of $\M$ from the sample $\XX_n\subseteq \RR^ D$ in a setting in which both $\M$ and $f$ are assumed to be unknown. 
A standard approach to accomplish this task consists in applying to  $\XX_n$ a computational technique known as \textit{persistent homology}, which allows to obtain qualitative information about connected components, cycles, voids and higher dimensional holes from the point cloud. Here, the sample $\XX_n$ is considered as a metric space endowed with some computable distance, such as the Euclidean distance or an estimator of the inherited geodesic distance.
Although the topological information carried by $\M$ remains the same when endowed with any Riemannian metric, the output of the application of persistent homology to $\XX_n$ strongly depends on the particular distance function employed.
In this article, we consider a computable estimator defined over $\XX_n$ of a certain Riemannian metric on
$\M$ that takes into account the density $f$, which was called \textit{Fermat distance} \cite{GJS}.
We show that the use of this density-based intrinsic metric in the computation of persistent homology can lead to results that overcome simultaneously certain weaknesses of standard approaches, such as the sensitivity to outliers and the dependence on the embedding of the sample in the ambient space.

Persistent homology is a central technique in Topological Data Analysis (TDA) developed to infer the \textit{homology groups} of a space by studying a sample $\XX_n$ at \textit{all} scales of resolution at the same time \citep[see][]{EH, ELZ, BCY, NSW, MR2121296}. 
It has found applications in many fields, including neuroscience \citep{Giusti}, finance  \citep{GK}, signal processing \citep{PH, TP, P}, computational neural networks \citep{pmlr-v108-gabrielsson20a}, virus evolution \citep{Chan18566} and sensor networks \citep{Silva2007CoverageIS}.
This method yields as output an object called {\em persistence diagram} associated to the sample. Under mild conditions, the homology groups of the underlying topological space can be read off the persistence diagram \cite[see][]{ELZ}. In \cite{CCSGGO, CDSGO}, Chazal et.al provided a general framework that allows to define persistence diagrams for infinite metric spaces instead of just finite approximations (samples). Thus, one can view the persistence diagram associated to a sample of a space as an estimate of a limiting object, namely, the persistence diagram of the entire space. When the distinction is needed, we will call these diagrams {\em sample persistence diagram} and {\em population persistence diagram} respectively.

Our main result states that, under reasonable conditions, there is convergence as metric spaces of the sample $\XX_n$ endowed with a computable estimator of the Fermat distance towards the manifold $\M$ (equipped with the Fermat distance) in the sense of Gromov--Hausdorff as the size $n$ grows. 
When combined with the well-known {\em stability theorem}  \citep{CSEH, CDSO, CCSGGO}, this approximation result as metric spaces allows to deduce the convergence of the corresponding persistence diagrams. For this purpose, the space of diagrams is naturally equipped with the \textit{bottleneck distance}.
Approximation results that include convergence rates and confidence regions have been established when the metric of the target space is known; see e.g. \citep{FLRWBS} where the Euclidean distance is considered for both the samples and the space, and also \citep{CGLM} where a general metric is used but assumed to be known in advance.

Persistence diagrams are known to be sensitive to the presence of outliers \cite[see][]{MR2854318, chazal2011geometric, buchet2016efficient, MR3968644}. In \cite{chazal2011geometric, MR3968644}, the authors proposed filtrations of point clouds regarded as empirical measures in the ambient Euclidean space --- called DTM-filtrations --- to achieve a robust computation of ambient persistent homology. This theory was later extended to general metric spaces
in \cite{buchet2016efficient}. On the other hand, intrinsic versions of the classical \v{C}ech and Vietoris--Rips filtrations were developed with the aim of capturing topological properties of manifolds sitting in an Euclidean space which are independent of the embedding. The approach exhibited in this article handles both difficulties at the same time. Indeed, we show that sample persistence diagrams computed using the estimator of the (intrinsic) Fermat distance are both robust to outliers for positive degree and display the correct homology of the manifold for a longer parameter interval as compared with the use of ambient Euclidean distance.
\par We refer the reader to the video \cite{tutorial} for an introductory exposition of the contents of this article.

\subsection{Contributions}  Let $(\M, \rho)$ be a smooth $d$-dimensional Riemannian manifold embedded in $\RR^D$ with density $f:\M\to \RR_{>0}$ and a Riemannian density-based distance $\rho$ (mainly, it will be the Fermat distance $d_{f,p}$ defined below).

\par For $p>1$, the {\em population Fermat distance} is defined as 
\begin{equation*}
\label{def.fermat}
d_{f,p}(x,y) = \inf_{\gamma} \int_{I}\frac{1}{f(\gamma_t)^{(p-1)/d}}|\dot{\gamma_t}| dt.
\end{equation*}
Here $x, y\in \M$, $|\cdot|$ denotes the Euclidean distance and the infimum  is taken over all piecewise smooth curves $\gamma  \colon I=[0,1] \to \M$ with $\gamma(0) = x$, and $\gamma(1) = y$.
In the special case when $f$ is uniform, the population Fermat distance reduces to (a multiple of) the  inherited Riemannian distance $d_{\M}$ from the ambient Euclidean space. When this is not the case, this distance takes into account the density, which may be advantageous in certain situations, like in the case of estimation of the topology of $\M$ from samples with presence of noise and outliers. This metric was also considered in the works \cite{HDH, MCD, SGJ,GJS}.

\par Given a finite set of points $\mathbb X_n$, the \textit{sample Fermat distance} between $x,y$ is defined as
\begin{equation*}
\label{def.fermat.sample}
d_{\XX_n, p}(x,y) = \inf_{\gamma} \sum_{i=0}^{r}|x_{i+1}-x_i|^{p}
\end{equation*}
where the infimum is taken over all paths $\gamma=(x_0,x _1, \dots, x_{r+1})$ with $x_0=x$, $x_{r+1} = y$ and $\{x_1, x_2, \dots, x_{r}\}\subseteq \XX_n$. 

Our main result states the Gromov--Hausdorff convergence (a.s.) of the sample endowed with the sample Fermat distance, appropriately re-scaled, to $(\M, d_{f,p})$.

\smallskip

\noindent
{\bf Theorem }  {\it Let $\M$ be a smooth, closed $d$-dimensional Riemannian manifold embedded in $\RR^D$. Let $f:\M\to \RR_{>0}$ be a smooth density function. Let $\XX_n = \{x_1, x_2, \dots, x_n\}\subseteq \M$ be a set of $n$ independent sample points in $\M$ with common density $f$. 
Given $p>1$, there exists a constant $\mu = \mu(p,d)$ such that for every $\lambda \in \big((p-1)/pd, 1/d\big)$ and $\varepsilon>0$ there exist $\theta>0$ satisfying
\[
    \PP\left( d_{GH}\left(\big(\M, d_{f,p}\big), \big(\XX_n, {\scriptstyle \frac{n^{(p-1)/d}}{\mu}} d_{\XX_n, p}\big)\right) > \varepsilon \right) \leq \exp{\left(-\theta n^{(1 - \lambda d) /(d+2p)}\right)}
\]
for $n$ large enough, where $d_{GH}$ stands for the Gromov-Hausdorff distance between metric spaces.}

\smallskip

As a consequence of this result and the stability theorem for persistence diagrams we deduce the following convergence result.

\smallskip

\noindent
{\bf Corollary }  {\it
Let $\varepsilon > 0$ and $\lambda \in \big((p-1)/pd, 1/d\big)$. There exists a constant $\theta > 0$ such that
\[
    \PP\Big( d_b\big(\dgm(\Filt(\mathcal M, d_{f,p})),\dgm(\Filt(\XX_n, {\scriptstyle \frac{n^{(p-1)/d}}{\mu}} d_{\XX_n, p}))\big) > \varepsilon \Big) \leq \exp{\left(-\theta n^{(1 - \lambda d) /(d+2p)}\right)}
\]
for $n$ large enough.
}

\smallskip

Here $\Filt(\cdot)$ denotes either the Vietoris--Rips or \v{C}ech filtration, $\dgm(\cdot)$ the associated persistence diagram and $d_b$ is the bottleneck distance (see Section \ref{persistence_diagrams} for precise definitions). Since $(\mathcal M, d_{f,p})$ is a Riemannian manifold, its population persistence diagram $\dgm(\Filt(\mathcal M, d_{f,p}))$ displays the correct homology up to the convexity radius $\convex(\M,d_{f,p})$. In contrast, for $(\M, |\cdot|)$ this is guaranteed only up to the reach $\tau_{\M}$. It is easy to find examples of manifolds in which $\convex(\M,d_{f,p})$ is much larger than $\tau_{\M}$.
\par On the other hand, we prove that for a reasonable upper bound $r$ on the filtration parameter, $\dgm(\Rips_{<r}(\XX_n, d_{\XX_n, p}))$ is robust to outliers for homology degree greater than 0. 

\smallskip

\noindent
{\bf Proposition }  {\it Let $\XX_n$ be a sample of $\M$ and let $Y\subseteq \RR^D\smallsetminus \M$ be a finite set of outliers. Let $\delta = \displaystyle \min\Big\{\min_{y\in Y} d_E(y, Y\smallsetminus \{y\}), ~d_E(\XX_n, Y)\Big\}$, where $d_E$ denotes the Euclidean distance between sets.
For all $k>0$ and $p>1$,
 \[
 \dgm_k(\Rips_{<\delta^p}(\XX_n \cup Y, d_{\XX_n\cup Y, p})) = \dgm_k(\Rips_{<\delta^p}(\XX_n, d_{\XX_n, p})),
\]
where $\Rips_{<\delta^p}$ stands for the Rips filtration  up to parameter $\delta^{p}$ and $\dgm_k$ for the persistent homology of degree $k$.
}

\smallskip

The threshold $\delta^p$ is restrictive if it is below $\diam(\XX_n, d_{\XX_n, p})$. However, we will show that under a natural model for the outliers, $\delta^p > \diam(\XX_n, d_{\XX_n, p})$ for large enough $p$.

\subsection{Applications to Signal Analysis.} The study of time series  --- specially, derived from dynamical systems --- through the inference of homology groups of a certain associated space called \textit{delay embedding} was pioneered in the works \cite{P, PH}. The construction of the delay embedding of a time series heavily depends on the dimension or number of independent variables of the underlying system, and the choice of a parameter called \textit{time delay.} It often leads to analyse subspaces of a sufficiently high dimensional Euclidean space, which makes the inference of topological information unstable.
\par In first place, by means of concrete examples involving the Lorenz attractor and noisy periodic signals, we show that the use of Fermat distance in this method can lead to a more robust inference of the delay embeddings' topological features. The reason behind this is that the Fermat distance is less prone, compared to the Euclidean distance, to the effect known as \textit{curse of dimensionality} and less dependent on the particular embedding. We also
describe a method to detect change-points in the time series through the study of the evolution in time of the persistence diagrams of the corresponding time-delay embeddings. This is applied to discover anomalies in electrocardiogram signals and different patterns in the song of canaries corresponding to different syllables. 

The code to replicate the computational examples and applications can be found at the repository \cite{repository}.

\subsection{Related Work}
The sample Fermat distance was introduced independently in the articles \cite{SGJ, MCD}. The study of approximations of density based metric from samples was suggested in \cite{VB} and developed in \cite{SO}. In \cite{cohen15approximating, CMS} it was analyzed a general family of metrics that includes the population Fermat distance and deeply studied the case $p=2$ of sample Fermat distance, which was also called power weighted shortest distance in \cite{MCD}.
\cite{GJS} proved that it is possible to recover the population Fermat distance $d_{f,p}$ for $d$-dimensional manifolds which are isometrically embedded (closures of) open sets of $\RR^{d}$ in $\RR^D$ as the limit of the sample Fermat distance. In the related work in \cite{HDH} it was shown that in the same context, a statistic that is similar to the sample Fermat distance but uses the inherited Riemannian distance $d_{\M}$ between consecutive points in a path instead of the Euclidean one to measure its cost, also converges almost surely to the Fermat distance. We remark that this statistic cannot be computed from the sample since the inherited distance is not assumed to be known in advance. However, the results in \cite{HDH} provides an essential and strong  foundation on the basis of which our main result is built over.

The problem of learning geodesic distances from samples for submanifolds of the Euclidean space, specially with the aim of reducing dimensionality and visualizing data, has a long history; see for instance \citep{Isomap, MR2147326}. On the other hand, the problem of estimating the persistence diagram of a submanifold of an Euclidean space from a sample has been studied in \cite{FLRWBS, CGLM}, where the underlying metric is assumed to be known. In this setting, both works \cite{CGLM} and \cite{FLRWBS} were able to prove the following satisfying result: the persistence diagrams computed using the sample converge almost surely (in the sense of bottleneck distance) to the persistence diagram of the desired metric space. Moreover, they gave exponentially small bounds in the size of the sample for the probability of the bottleneck distance between the corresponding persistence diagrams being larger than some positive number; see \cite[Corollary 3]{CGLM} and \cite[Lemma 4]{FLRWBS}, where in addition confidence sets for persistence diagrams are provided. In a different direction, the advantages of computing persistence diagrams of submanifolds of an Euclidean space using alternative metrics --- more specifically, metrics based on diffusion geometry and random walks --- were explored experimentally in \cite{MR2854318}.

\subsection{Structure of the Paper}
In Section \ref{distance learning} we prove our main result Theorem \ref{thm GH} regarding the Gromov--Hausdorff convergence of metric spaces using, respectively, the sample and the population Fermat distance. Section \ref{persistence_diagrams} includes an introduction to  persistent homology and is devoted to the study of persistence diagrams of manifolds endowed with Fermat distance. We deduce in first place the convergence of sample persistence diagrams to population persistence diagrams. Then, we show that by using these intrinsic metrics the topological features last longer in the persistence diagrams. Finally, we show that Fermat-based persistence diagrams are robust to the presence of outliers for positive homology degree. In Section \ref{experiments} we present a method for pattern recognition in time series, which is applied to real data from electrocardiograms and songs of canaries. Appendix \ref{proofs} contains the proofs of some technical results (Proposition \ref{thm resta} and Lemma \ref{spacing}), required as intermediate steps to prove Theorem \ref{thm GH}.

\section{Density-based Distance Learning}\label{distance learning}

In this section we prove the main theorem of the article, which states that the sample $\XX_n$, considered as a metric space with the sample Fermat distance (appropriately re-scaled), converges almost surely to $(\M, d_{f,p})$ in the sense of Gromov--Hausdorff.

\par
We begin by introducing the \textit{population Fermat distance} for a smooth closed Riemannian manifold without boundary $\M$ of dimension $d >1$ with Riemannian metric tensor $g$ together with a positive $C^{\infty}$ density function $f:\M\to \RR_{>0}$. For $p >1$, consider the deformed metric tensor  $g_p=f ^{2 (1-p)/d}g$ given by a conformal transformation of the original metric $g$. Since $f$ is smooth, $g_p$ is a Riemannian metric tensor. Thus, $\M$ has a metric space structure given by the geodesic distance with respect to $g_p$, denoted by $d_{f,p}$. 

\begin{definition}\citep{HDH}
For $p>1$, the \textit{population Fermat distance} between $x,y\in \M$ is defined as
\begin{equation*}\label{eq:fermat.distance}
d_{f,p}(x,y) = \inf_{\gamma} \int_{I}\frac{1}{f(\gamma_t)^{(p-1)/d}}\sqrt{g(\dot{\gamma}_t, \dot{\gamma}_t)} dt 
\end{equation*}
where the infimum  is taken over all piecewise smooth curves $\gamma  \colon I\to \M$ with $\gamma_0 = x$, and $\gamma_1=y$.
\end{definition}

Notice that geodesics in $\M$ with respect to the distance $d_{f,p}$ are more likely to lie in regions with high values of $f$.
The name \textit{Fermat distance} comes from the analogy with optics, in which $d_{f,p}$ is the optical distance as defined by Fermat's principle when the refraction index is given by $f^{-(p-1)/d}$.

Consider now a set $\XX_n = \{x_1, x_2, \dots, x_n\}\subseteq \M$ of $n$ sample points in $\M$ with common density $f$. Suppose that $\M$ is embedded in $ \RR^D$ and it is endowed with the standard inherited Riemannian metric.
Our aim is to approximate $d_{f,p}(x,y)$, assuming no knowledge about $\M$ and the Riemannian distance defined on it. To achieve this, we will define an estimator for this distance over the sample. We denote by $|x-y|$ the Euclidean distance between points $x,y\in \M$.

\begin{definition}\citep{SGJ, MCD}
For $p> 1$, the \textit{sample Fermat distance} between $x,y\in \M$ is defined as
\[
d_{\XX_n, p}(x,y) = \inf_{\gamma} \sum_{i=0}^{r}|x_{i+1}-x_i|^{p}
\]
where the infimum is taken over all paths $\gamma=(x_0,x _1, \dots, x_{r+1})$ of finite length with $x_0=x$, $x_{r+1} = y$ and $\{x_1, x_2, \dots, x_{r}\}\subseteq \XX_n$. 
\end{definition}

Since $p>1$, geodesics with respect to this distance are also likely to lie in regions with high density of points in $\mathbb X_n$. This is due to the fact that paths with short edges are favored even if they have large total (Euclidean) length.

We remark here that, for technical reasons, we adopt a slightly different definition for the sample Fermat distance than the original one from \cite{SGJ}. Namely, in the original setting, only paths completely contained in $\XX_n$ are considered, including the endpoints. Points that are not in the sample $\mathbb X_n$ are projected to the nearest point in $\mathbb X_n$. In consequence, our sample Fermat distance here does not generally induce a pseudometric over $\M$, but only a metric when restricted to $\XX_n$.

\begin{example}[Eyeglasses]\label{eyeglasses example} The effect of taking different values of $p$ for the sample Fermat distance $d_{\mathbb X_n,p}$ in the geometry of a manifold is illustrated below. Concretely, the \textit{eyeglasses} curve in $\RR^2$ uniformly sampled and perturbed with Gaussian noise is considered (see Figure \ref{fig:eyeglasses dist}). We compute the sample Fermat distance between each pair of points for a series of values of $p>1$ and embed the sampled points in $\RR^{2}$ in such a way that the Euclidean distance in the embedding reflects the Fermat distance, using the Multidimensional Scaling algorithm (MDS). As $p$ becomes larger, the geometry of the data overcomes the bottleneck region and it deforms into a circle. We also compute the Isomap embedding in $\RR^2$ posed in \cite{isomapproofs}. Recall that the Isomap embedding is the MDS projection with an estimator of the inherited Riemannian distance based in the $k$-NN graph as input distances \cite[see][Section 5] {isomapproofs}.
Due to the noise near the bottleneck region, some points that are far in the sense of the inherited Riemannian distance become close in the distance estimated from the $k$-NN graph.  Note that Isomap embedding is highly sensitive to noise, while with Fermat distance the points lying in low density regions are mapped to points that are far from the rest of the sample. The larger the power $p$, the stronger this effect. This feature allows Fermat distance to reconstruct the underlying topology of the manifold in the present case, even with noise, for a range of values of $p$.

\begin{figure}[tb!]
    \centering
    \includegraphics[width=0.37\textwidth]{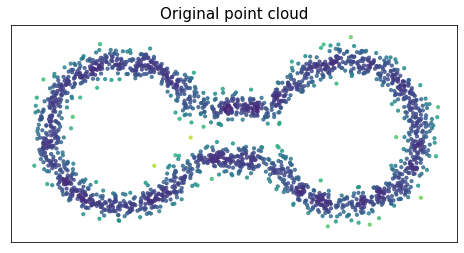}
    \includegraphics[width=0.28\textwidth]{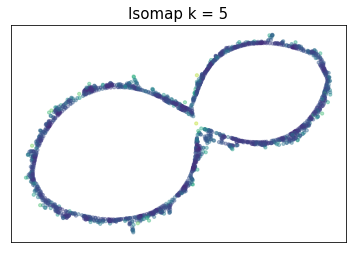}\\
    \includegraphics[width=0.32\textwidth]{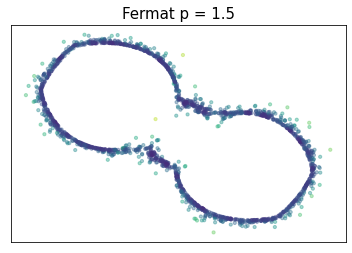}
    \includegraphics[width=0.32\textwidth]{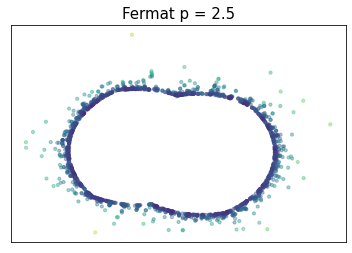}
    \includegraphics[width=0.32\textwidth]{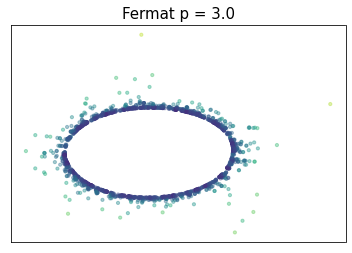}

    \includegraphics[width=0.328\textwidth]{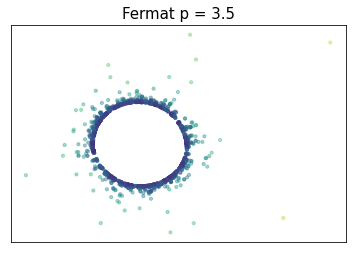}
    \includegraphics[width=0.328\textwidth]{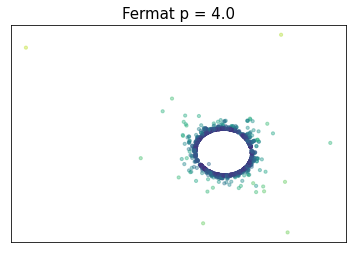}
    \includegraphics[width=0.328\textwidth]{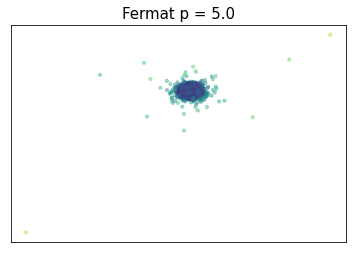}
    \caption{Top: A sample  with noise of 2000 points of the eyeglasses dataset and Isomap projection with $k=5$ (similar results are obtained for all reasonable values of $k$). Points are coloured according to local density. Middle and bottom: MDS embedding in $\RR^2$ using Fermat distance for different values of $p$. }
    \label{fig:eyeglasses dist}
\end{figure}
\end{example}

\begin{remark}[The role of $p$]  
The parameter $p$ in the definition of the population Fermat distance $d_{f,p}$ controls the density weight  $f^{-(p-1)/d}$ in the computation of geodesics. Whereas for $p=1$  the optimal paths are obtained in classic geodesic paths, for large $p$ they might significantly differ, being mostly restricted to areas of high density.
In practice, the value of $p$ in the sample Fermat distance $d_{\XX_n, p}$ quantifies the balance between the embedding and the density of a given sample $\XX_n$ when estimating the optimal paths (notice that it is equivalent to the Euclidean distance for $p=1$). In general, there is a reasonable large --- although bounded --- interval of values of $p$ for which the estimator $d_{\XX_n, p}$  allows to recover the intrinsic geometry of the sample $\XX_n$ even in presence of noise (c.f. Example \ref{eyeglasses example}). A similar phenomena can be experimentally observed when it is used in clustering tasks, as shown in simulations in \cite{SGJ} and \cite{LMM}.
\end{remark}

\begin{remark}[Dimensionality reduction]  The estimation of Fermat distance on input data, when coupled with the MDS projection, produces a new method to achieve dimensionality reduction. 
This strategy is in analogy with the popular algorithm Isomap \cite{Isomap}. 
It is known that Isomap suffers from topological instability in presence of noise, since it may construct erroneous connections (called \textit{short-circuits}) in the $k$-NN graph that potentially impair its performance (see \cite{Balasubramanian2002TheIA}). In contrast, since noise generally corresponds with regions of low density, noisy points are treated by our method almost as not being part of the manifold. These effects increase with the value of $p$, and they might be advantageous for the inference of the right geometry of the data (c.f. Section \ref{section outliers}).
\end{remark}

\par Our first result, Proposition \ref{thm resta}, 
shows that the sample Fermat distance converges to the population Fermat distance for closed (i.e. compact and without boundary) submanifolds of $\RR^{D}$. A related result was previously proved in \cite{GJS} for isometrically embedded (closures of) open sets of $\RR^{d}$.  Here we extend the class of manifolds to any compact manifold without boundary embedded in $\RR^D$. Moreover, Proposition \ref{thm resta} states a \textit{uniform} convergence  for \textit{any} two points in the manifold --- not only pointwise, as stated in \cite{GJS} ---. This feature is essential to study both the manifold and the sample endowed with the (population and sample respectively) Fermat distance as \textit{single} objects (metric spaces) and to 
prove convergence in the sense of Gromov--Hausdorff.

Let us fix some notations and general hypotheses. Hereafter, $\M$ will denote a smooth $d$-dimensional closed Riemannian submanifold of $\RR^D$ endowed with the inherited Riemannian distance $d_{\M}$. We will consider a set $\XX_n\subseteq\M$ of $n$ independent random points with common smooth density $f\colon\M\to \RR_{>0}$. We will denote by $M_f$ and $m_f$ the maximum and minimum values attained by $f$ on $\M$, respectively. Observe that $0<m_f<M_f<\infty$. Finally, given $p>1$ we set $\alpha = 1/(d+2p)$.

\begin{proposition}\label{thm resta}
For every $p>1$ and $\lambda \in \big((p-1)/pd, 1/d\big)$, given $\varepsilon>0$ there exist $\mu, \theta>0$ such that
\[
    \PP\left(\sup_{x,y}\left| n^{(p-1)/d} d_{\XX_n, p}(x,y) - \mu d_{f,p}(x,y )\right|>\varepsilon\right) \le \exp\left(-\theta n^{(1 - \lambda d)\alpha}\right)
\]
for $n$ large enough. The supremum is taken over  $x,y\in \M$.
\end{proposition}

The constant $\mu$ from the statement is fixed throughout this manuscript and depends only on $p$ and $d$. It was originally defined in \cite[Lemma 3]{HN}. The constant $\theta$ depends on $\epsilon, p, f$ and $\M$.

Proposition \ref{thm resta} is derived from a related result in \cite{HDH}, in which the authors establish the convergence of a sample statistic known as the \textit{power-weighted shortest path} to the population Fermat distance. For $p > 1$ and points $x,y\in \M$, the power-weighted shortest path between $x,y$ is defined as
\begin{equation}\label{weighted shortest path}
L_{\XX_n, p}(x,y) = \inf_{\gamma}\sum_{i=0}^{k}d_{\M}(x_{i+1},x_i)^{p}
\end{equation}
where the infimum is taken over all paths $\gamma=(x_0, \dots, x_{k+1})$ in $\XX_n$ of finite length with $x_0=x$, $x_{k+1} = y$.

\begin{theorem}\cite[Theorem 1]{HDH}\label{thm damenlin}
Let $p>1$ and $\varepsilon > 0$. Suppose that $(b_n)_{n \geq 1}$ is a sequence of positive real numbers such that $\frac{\log(n)}{n b_n^{d}} \to 0$ as $n$ goes to infinity. Then, there exists a constant $\theta>0$ (which depends on $\varepsilon$) such that
\[
\PP \left(\sup_{\substack{x,y\in \M\\d_{\M}(x,y)\geq b_n}}\left|\frac{n^{(p-1)/d}L_{\XX_n, p}(x,y)}{ d_{f,p}(x,y)}-\mu\right|>\varepsilon\right)\leq \exp(-\theta(n b_n^{d})^{\alpha})
\]
for all sufficiently large $n$, where the supremum is taken over $x,y\in \M$ with
$d_{\M}(x,y)\geq b_n$.
\end{theorem}

As explained in the paragraph following Theorem 1 in \cite[p. 2793]{HDH}, the requirement that $\frac{\log(n)}{n b_n^{d}} \to 0$ is necessary in order to obtain a nontrivial upper bound for the probability.

Note that in Theorem \ref{thm damenlin}, the convergence holds for the set of points $x,y\in \M$ with $d_{\M}(x,y)$ greater than some sequence $(b_n)$. However, since we will be interested in studying the Gromov--Hausdorff convergence of the associated metric spaces (see \eqref{GH_raw} below), it is necessary to
have uniform control of the convergence of the estimated distance for \textit{all} points in the manifold.
The uniform convergence is one of the main improvements upon Theorem \ref{thm damenlin} we show in Proposition \ref{thm resta}.
Also, notice that the proposed statistic $L_{\XX_n, p}$ of $d_{f,p}$ is based on the previous knowledge of the inherited Riemannian distance $d_\M$. In the general data analysis setting, only a sample of points in a Euclidean space is given. Under the assumption that points lie on an (unknown) manifold $\M$, the goal is to find an estimator of the intrinsic distance $d_{f,p}$ that can be completely computed from the sample. In Proposition \ref{thm resta}, we prove that sample Fermat distance $d_{\XX_n, p}$ is indeed a good estimator of $d_{f,p}$. 

Proposition \ref{thm resta} arises
as a natural continuation of Theorem \ref{thm damenlin}.
The main idea of the proof is to show that any segment that is part of any shortest path with respect to $d_{\XX_n,p}$ will be arbitrarily small with high probability if $n$ is large enough.
This will allow us to deduce that the power-weighted distance is well approximated by the sample Fermat distance.
We defer the proof to Appendix \ref{proofs}.

We will next  estimate the Gromov--Hausdorff distance between the metric space $\XX_n$ with an appropriate re-scaling of the sample Fermat distance $d_{\XX_n, p}$ and $\M$ endowed with the population Fermat distance $d_{f,p}$. Recall that the \textit{Gromov--Hausdorff distance} $d_{GH}$ is a metric on the (isometry classes of) compact metric spaces that, roughly speaking, quantifies how difficult it is to match every point of a metric space $(\XX,\rho_{\XX})$ with some point of another space $(\YY, \rho_{\YY})$. More formally, it is defined as
\begin{equation}\label{GH_raw}
   d_{GH}\big((\XX, \rho_\XX),(\YY, \rho_\YY)\big) := \inf \{d_H(h_1(\XX), h_2(\YY))\},
\end{equation}
where the infimum is over all the isometric embeddings $h_1\colon\XX \to \WW$, $h_2\colon\YY \to \WW$ in a common metric space $\WW$ and $d_H$ stands for the Hausdorff distance. We will employ the following equivalent characterization of the Gromov-Hausdorff distance, which is often more convenient:
\begin{equation}\label{GH}
d_{GH}\big((\mathbb X, \rho_{\mathbb X}) ,(\mathbb Y, \rho_{\mathbb Y})\big) = \frac{1}{2}\inf_R \sup_{(x,y), (x', y')\in R} |\rho_\XX(x, x') - \rho_\YY(y, y')|,
\end{equation}
where the infimum is taken over subsets $R\subseteq \mathbb X\times \mathbb Y$ such that the projections $\pi_{\mathbb X}(R) = \mathbb X$, $\pi_{\mathbb Y}(R) = \mathbb Y$.
 
We are now ready to state our main theorem. For notational convenience, we set $d_{n,p} = \frac{n^{(p-1)/d}}{\mu} d_{\XX_n, p}$, the re-scaled sample Fermat distance on $\XX_n$.

\begin{theorem}\label{thm GH}
Let $\varepsilon > 0$ and $\lambda \in \big((p-1)/pd, 1/d\big)$. There exists a constant $\theta > 0$ such that
\[
    \PP\big( d_{GH}( (\M, d_{f,p}), (\XX_n, d_{n,p})) > \varepsilon \big) \leq \exp{\left(-\theta n^{(1 - \lambda d) \alpha}\right)}
\]
for $n$ large enough and $\alpha = 1/(d+2p)$.
\end{theorem}

Before presenting the proof of Theorem \ref{thm GH}, we will need a preliminary lemma which asserts that, with high probability, no point of $\M$ is too far from the nearest point of the sample. The argument of this proof is standard, but we include it in Appendix \ref{proofs} for the reader's convenience.

\begin{lemma}\label{spacing}
For any $\kappa > 0$, the event
\[
    \left\{ \sup_{x \in \M} d_{\M}(x,\XX_n) \geq  n^{(\kappa-1)/d} \right\}
\]
holds with probability at most $\exp(-\theta n^{\kappa})$ for some constant $\theta > 0$ if $n$ is large enough.
\end{lemma}

We are now in position to prove Theorem \ref{thm GH}.

\begin{proof} [Theorem \ref{thm GH}] In order to compute the Gromov--Hausdorff distance between $(\M, d_{f,p})$ and $ (\XX_n, d_{n,p})$, we consider in \eqref{GH} the relation 

\[
R=\{(x_i, x_i)\colon x_i\in \XX_n\}\cup\{(x_{y}, y)\colon y \in \mathcal M, d_{f,p}(x_{y}, y) = d_{f,p}(\XX_n, y)\}.
\]

By a simple application of the triangle inequality we get that 
\begin{equation}
\label{GH bound}
d_{GH}\big( (\M, d_{f,p}), (\XX_n, d_{n,p})\big) \le \frac{1}{2}\left(\sup_{x,y\in \XX_n} |d_{f,p}(x, y) - d_{n,p}(x, y)| + 2 \sup_{y\in \mathcal M}d_{f,p}(\XX_n, y) \right).
\end{equation}

Observe that the two terms on the right hand side of the previous inequality can be bounded above by Proposition \ref{thm resta} and Lemma \ref{spacing} respectively. 

Given $\varepsilon>0$, by   \eqref{GH bound}
we have that 
\begin{align*}
\PP&\Big(d_{GH}\big((\M, d_{f,p}), (\XX_n, d_{n,p})\big)>\varepsilon/2\Big) \\
& \le \PP\left(\sup_{x,y\in \XX_n} |d_{f,p}(x, y) - d_{n,p}(x, y)|>\varepsilon/2 \right) + \PP\left(\sup_{y\in \mathcal M}d_{f,p}(\XX_n,y)>\varepsilon/4\right)\\
\end{align*}
To bound the first term, we apply Proposition \ref{thm resta} to get
\[
\PP\left(\sup_{x,y\in \XX_n} |d_{f,p}(x, y) - d_{n,p}(x, y)|>\varepsilon/2 \right)\leq \exp\big(-\theta n^{(1 - \lambda d) \alpha}\big).
\]
for some positive constant $\theta$ and $n$ sufficiently large. As for the second term, notice that since
\[
    d_{f,p}(x,y)\leq m_f^{-(p-1)/d} d_{\M}(x,y),
\]
Lemma \ref{spacing} implies
\[
    \PP\left(\sup_{y\in \mathcal M}d_{f,p}(\XX_n,y)> n^{(\alpha-1)/d}m_f^{(p-1)/d}\right)\leq \exp(-\theta n^\alpha)
\]
for $n$ large. The proof follows by noticing that the sequence $n^{(\alpha-1)/d}m_{f}^{-(p-1)/d}$ converges to $0$ as $n$ goes to infinity.
\end{proof}

\begin{remark}[Rate of convergence]
The rate of convergence in Theorem \ref{thm GH} is related to the fluctuations of $n^{\frac{p-1}{d}}d_{\mathbb X_n,p}(x,y)$ around $\mu d_{f,p}(x,y)$ or, more coarsely, the variance of $n^{\frac{p-1}{d}}d_{\mathbb X_n,p}(x,y)$ (\cite{DW} provides strong evidence that the bias can be bounded by the variance). It is expected that this variance decreases as a power of $n$, i.e. 
\[
c n^{-\zeta} \le \mathrm{Var} \left( n^{\frac{p-1}{d}}d_{\mathbb X_n,p}(x,y) \right) \le Cn^{-\zeta}
\] for a dimension-dependent constant $\zeta = \zeta(d)>0$.
The precise value of $\zeta (d)$ is a still open problem in probability theory in the context of 
First Passage Percolation (\cite{HN2,ADH}). For $d=1$ it can be proved 
that $\zeta=1$.
For $d\ge 2$ it is widely believed \citep{ADH} that the exponent should not depend on $p$ and that for $d=2$ we should have $\zeta(2)=2/3$. 
For $d\ge 3$ it is not clear what the value of $\zeta(d)$ should be. If we write $\zeta(d) = -2(\chi(d) -1)/d$, it is expected that $\chi(d)$ should decrease with the dimension but there is not agreement on whether there exists some critical dimension $d_c$ such that $\chi(d)=0$ for $d\ge d_c$ or even if we should have $\chi(d) \to 0$ as $d \to \infty$ \cite[Section 3]{ADH}. In \cite{HN2} non-optimal rigorous bounds have been proven for Euclidean First Passage Percolation that in our context read
\[
    \PP\left( d_{GH}\left(\big(\M, d_{f,p}\big), \big(\XX_n, {\scriptstyle \frac{n^{(p-1)/d}}{\mu}} d_{\XX_n, p}\big)\right) > n^{-\frac1 d + \varepsilon} \right) \leq C_1 \exp{\left(-C_0 n^{\varepsilon}\right)}
\]
for positive constants $C_0, C_1$ depending on $\varepsilon>0$.
This bound follows immediately in our case when $\M$ is the closure of a bounded open and convex set and $f$ is constant on $\M$. For the general case considered in this manuscript we expect to have similar bounds. Obtaining those bounds would be highly valuable, but its analysis is out of the scope of this paper.  We refer the reader to \cite{LMM} for a detailed discussion about the rate of convergence.
\end{remark}

\section{Fermat-based Persistent Homology}\label{persistence_diagrams}

\par In this section we explore the use of Fermat distance as input in the computation of the persistence diagram associated to a sample of a manifold. We deduce the almost sure convergence of persistence diagrams of the sample $\XX_n$  with the (re-scaled) sample Fermat distance  towards the persistence diagram of $(\M, d_{f,p})$. We also show that we expect to read the correct homology of $\M$  for a longer parameter interval in the diagram associated to the sample $\XX_n$ computed with Fermat distance as compared with the use of Euclidean distance. Finally, we prove that Fermat-based persistence diagrams are robust to the presence of outliers for homology degree greater than 0.

\subsection{Convergence of Persistence Diagrams}
We start by briefly recalling the main concepts and results in persistent homology theory and refer the reader to the works \cite{CDSO, CDSGO} for a more thorough exposition.

For the computation of the persistent homology of a point cloud, one imagines each point as a {\em ball} (that is, representing a small surrounding region) and builds a combinatorial model for the space connecting the points according to whether the corresponding regions intersect. More precisely, for every fixed value of a parameter or {\em scale} that controls the size of the region that each point represents, one gets a {\em simplicial complex} (i.e., a higher dimensional analogue of a graph). This family of simplicial complexes, also known as a {\em filtration}, is the input of the procedure to compute persistent homology. Indeed, the topological features of this family of complexes change as the scale parameter grows: different connected components join in one, some loops are filled, new cavities appear, etc. By analyzing these transitions, we are able to assign a \textit{birth} and a \textit{death} value to each of these features, and the difference between them represents its \textit{persistence}. The most persistent features represent \textit{topological signatures}, whereas the shortest intervals may be considered as \textit{noise}. The output of this procedure is summarized in an object called \textit{persistence diagram}. We next give the formal definitions.

\par Given a (possibly infinite) metric space $(\mathbb X,\rho)$, a filtration over the real numbers $\mathrm{Filt}(\mathbb X,\rho) \\ = (\mathrm{Filt}_{\epsilon}(\mathbb X, \rho))_{\epsilon\in \RR}$ is a family of simplicial complexes with vertex set $\mathbb X$ such that $\mathrm{Filt}_{\epsilon}(\mathbb X) \subseteq \mathrm{Filt}_{\epsilon'}(\mathbb X)$ whenever $\epsilon \le \epsilon'$. For the purposes of this article, we are going to consider only some natural filtrations that are strongly linked to the metric $\rho$.
The \textit{\v{C}ech filtration}  consists of  a family of  simplicial complexes $(\Cech_{\epsilon}(\XX))_{\epsilon\in \RR}$ where a set of points $[x_0, \dots, x_k]$ forms a $k$-simplex of $\Cech_{\epsilon}(\XX)$  if the intersection of the $k+1$ closed balls $\bar B_{\rho}(x_i, \epsilon)$ is non empty. Equivalently, $\Cech_{\epsilon}(\XX)$ is the \textit{nerve} of the cover $\{\bar B_{\rho}(x, \epsilon)\colon x\in \XX\}$.
The \v{C}ech complex is the most natural way to build a simplicial complex associated to a space, since in favourable cases, it allows to recover its homotopy type as a consequence of the Nerve Theorem \cite[\S 4.G]{Hat}. However, the construction of the \v{C}ech complex is expensive from a computational point of view, since it requires to check for a large number of intersections. 
To circumvent this issue, one can instead consider the \textit{Vietoris--Rips filtration} $(\Rips_{\epsilon}(\XX))_{\epsilon\in \RR}$.  The $k$-simplices of $\Rips_{\epsilon}(\XX)$ are sets   $[x_0, \dots, x_k]$  such that $\rho(x_i, x_j) \leq \epsilon$ for all $0\leq i,j\leq k$. Equivalently, $\Rips_{\epsilon}(\XX)$ can be defined as the flag complex of $\Cech_{\epsilon}(\XX)$ (that is, the clique complex of the 1-skeleton of $\Cech_{\epsilon}(\XX)$).
If $\XX$ is a subset of the Euclidean space $\RR^D$, then one have $\Cech_{\epsilon}(\XX)\subseteq \Rips_{2\epsilon}(\XX)\subseteq \Cech_{\sqrt {2D/(D+1)}\epsilon}(\XX)$; see e.g. Theorem 2.5. from \cite{MR2308949}. In this sense, the Rips complex is a computationally efficient approximation of the \v{C}ech complex. Other filtrations involving lower dimensional simplices, such as the \textit{Alpha filtration} \cite{ MR713690}, can also be considered in our context. 

\par For any filtration as above, it is clear that the topology of the complexes $\Filt_{\epsilon}(\XX)$ will typically change as $\epsilon$ increases. This evolution is appropriately captured by considering the homology groups (over a field $\mathbf{k}$) of the nested family of simplicial complexes. One gets in this way a sequence of vector spaces $(H_\bullet(\Filt_{\epsilon}(\XX)))_{\epsilon \in \RR}$, where the inclusions $\Filt_{\epsilon}(\XX)\subseteq \Filt_{\epsilon'}(\XX)$ induce canonical linear maps $H_\bullet(\Filt_{\epsilon}(\XX))\to H_\bullet(\Filt_{\epsilon'}(\XX))$ in homology. Under some  conditions, such as finiteness of $\XX$ \citep{ELZ, MR2121296}, this sequence can be decomposed as a direct sum of \textit{intervals} $ I[\epsilon_b,\epsilon_d ]$ defined as
\[
 0\xrightarrow{\hspace{7pt}0\hspace{7pt}} \cdots 
 \xrightarrow {\hspace{7pt}0\hspace{7pt}}  0 
 \xrightarrow {\hspace{7pt}0\hspace{7pt}}  \underbrace{\mathbf{k}
 \xrightarrow {\hspace{7pt}\mathbbm{1}\hspace{7pt}}  \cdots 
 \xrightarrow {\hspace{7pt}\mathbbm{1}\hspace{7pt}}  \mathbf{k}}_{[\epsilon_b, \epsilon_d]}
 \xrightarrow {\hspace{7pt}0\hspace{7pt}}   0 
 \xrightarrow {\hspace{7pt}0\hspace{7pt}}  \cdots 
 \xrightarrow {\hspace{7pt}0\hspace{7pt}}  0
\]

Every interval is determined by the \textit{birth} and \textit{death} parameters $\epsilon_b$ and $\epsilon_d$ respectively, and it can be interpreted as a \textit{topological feature} of $\XX$ with an associated \textit{lifetime} $\epsilon_d-\epsilon_b$  (note that $\epsilon_d$ may be infinite, in that case the feature has infinite lifetime). 
The (multi)set  of points $(\epsilon_b, \epsilon_d)$ is called the \textit{persistence diagram}  of $(\XX, \rho)$ and is denoted $\dgm(\Filt(\XX, \rho))$ (or simply $\dgm(\Filt(\XX))$ if $\rho$ is clear from the context). Persistence diagrams are contained in the half (extended) plane above the diagonal $\Delta=\{(x,y)\colon x=y\}.$ For technical reasons, the diagonal $\Delta$ is considered as part of every persistence diagram with infinite multiplicity. 
In \cite{CCSGGO, CDSGO,CDSO} it is proved that, within a more abstract persistent framework, it is possible to extend the definition of persistence diagrams to some cases where the sequence might not be interval-decomposable. In particular, it is shown in \cite{CDSO} that if $\XX$ is a compact metric space, for every value of $\epsilon$ at most a finite number of new topological features appear (even though the vector spaces $(H_\bullet(\Filt_{\epsilon}(\XX)))_{\epsilon \in \RR}$ may be infinite-dimensional) and hence  $\dgm(\Filt(\XX))$ is well-defined. Notice also that all the definitions can be extended to filtrations indexed over connected subsets of the real line.

\begin{example}[Eyeglasses] 
We compute the persistence diagram associated to the Vietoris--Rips filtration of the sample points from Example \ref{eyeglasses example}, Figure \ref{fig:eyeglasses dist}. 
We compare the results obtained with different distant choices: the Euclidean distance, the $k$-NN estimator of the inherited Riemannian distance for $k=4$ and $k=5$ and the sample Fermat distance for $p=2.5$ and $p=3$. We also considered a weighted Vietoris--Rips filtration derived by a DTM-function with parameters $m=0.01$ and $p=1$  (see \cite{MR3968644} and Remark \ref{DTM}).
The homology of the eyeglasses curve has one generator of $H_0$ and one generator of $H_1$. However, it can be noticed that for either Euclidean and $k$-NN distance for $k\geq 5$, the persistence diagram displays two salient generators for the first homology group $H_1$, which can be attributed to the small reach of the manifold. As it can be seen in Figure \ref{fig:eyeglasses noise}, smaller values of $k$ fail to capture the geometry of the eyeglasses manifold.
A similar situation is presented using the  Vietoris--Rips DTM-filtration.
Finally, for the Vietoris--Rips filtration using Fermat distance for different choices of $p$, the diagrams show accurately only one persistent generator for $H_1$. On the other hand, the number of noticeable connected components increases with $p$. This effect is caused by the presence of noisy points in regions of extremely low density, becoming isolated points (or outliers) as $p$ evolves (cf. Remark \ref{H0}). 

\begin{figure}[tb!]
\begin{center}
\includegraphics[width=0.327\textwidth]{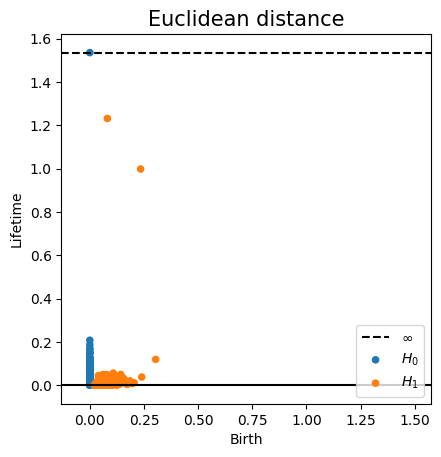}
\includegraphics[width=0.328\textwidth]{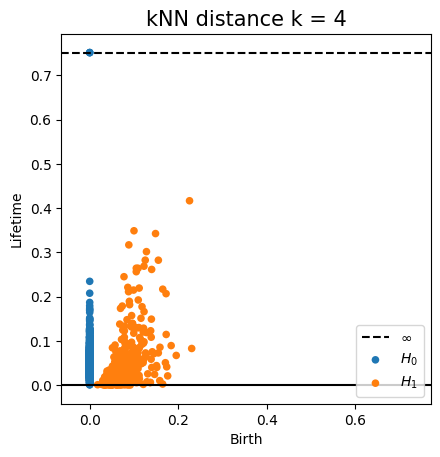}
\includegraphics[width=0.32\textwidth]{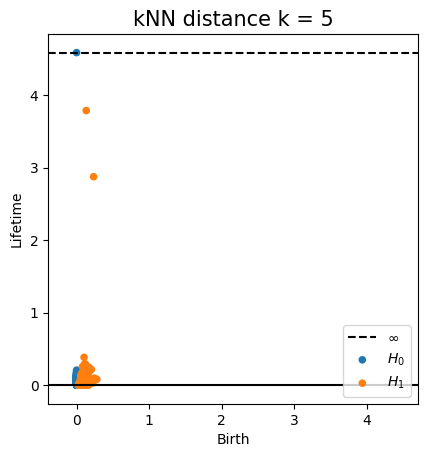}

\includegraphics[width=0.324\textwidth]{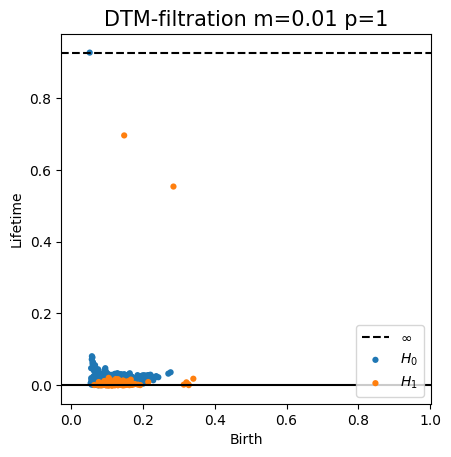}
\includegraphics[width=0.325\textwidth]{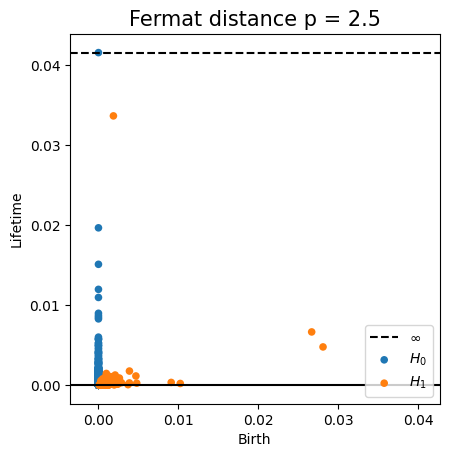}
\includegraphics[width=0.335\textwidth]{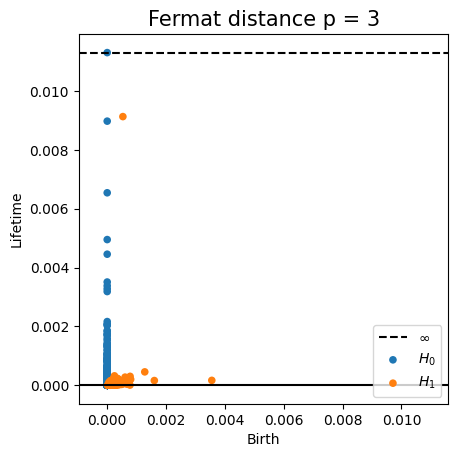}
\end{center}
\caption{Persistence diagrams (lifetime) associated to the eyeglasses point cloud with noise for different filtrations. Top: Vietoris--Rips filtration with Euclidean distance and $k$-NN distance for $k=4$ and $k=5$.
Bottom: Vietoris--Rips DTM-filtration with parameters $m=0.01$ and $p=1$ and Vietoris--Rips filtration with Fermat distance for $p=2.5$ and $p=3$.}
\label{fig:eyeglasses noise}
\end{figure}
\end{example}

Since in our setup we usually only get an approximation of the metric space under consideration, we will be interested in comparing persistence diagrams built on top of different metric spaces. In this sense, the {\em bottleneck distance} is a frequently used quantity to measure the difference between two persistence diagrams. Given persistence diagrams $\dgm_1$ and  $\dgm_2$, consider all perfect matchings $M\subseteq \dgm_1\times \dgm_2$ such that every point of $\dgm_1\smallsetminus \Delta$ and $\dgm_2\smallsetminus \Delta$ is paired exactly once in $M$. Note that points in $\dgm_1\smallsetminus \Delta$ and $\dgm_2\smallsetminus \Delta$ are allowed to be paired with points in the diagonal $\Delta$.
The bottleneck distance $d_b(\dgm_1, \dgm_2)$ is then defined as the infimum, over all such matchings $M$ as before, of the largest $\ell_{\infty}$-distance between matched pairs. That is,
\[
d_b(\dgm_1, \dgm_2) = \inf_{M} \max_{(x,y)\in M} |x-y|_{\infty}.
\]
\par The stability theorem \citep{CSEH, CDSO} ensures continuity (more precisely, Lipschitz continuity) in the process of computing persistence diagrams for a metric space. This means that small perturbations in the original metric space (in the sense of Gromov--Hausdorff) will translate into an at most proportional perturbation in the corresponding persistence diagram (in the sense of the bottleneck distance). 
Formally, it states that for any two precompact metric spaces $\mathbb X$ and $\mathbb Y$
\begin{equation}
\label{stability}
d_b\Big(\mathcal \dgm\big(\Filt(\XX, \rho_{\XX})\big),\mathcal \dgm\big(\Filt(\YY, \rho_{\YY})\big)\Big)\leq 2 d_{GH}\big((\XX,\rho_{\XX}),(\YY,\rho_{\YY})\big).
\end{equation}
This fact is exploited in \cite{FLRWBS, CGLM} to establish the almost sure convergence (in the sense of bottleneck distance) of the persistence diagrams associated to samples of a compact metric space drawn according to a measure satisfying certain hypotheses to the persistence diagram of the space. In these works the distance function of the underlying metric space is assumed to be known, and it is inherited by the sample.  
\par We are able to obtain convergence of persistence diagrams in our context, in which only an estimator of the underlying metric is available. Concretely, given the metric spaces $(\M, d_{f,p})$ and $(\XX_n, d_{n,p})$, from the estimation of its Gromov--Hausdorff distance  of Theorem \ref{thm GH} and the stability theorem \eqref{stability} we deduce the following result.

\begin{corollary}\label{conv.diag}
Let $\varepsilon > 0$ and $\lambda \in \big((p-1)/pd, 1/d\big)$. There exists a constant $\theta > 0$ such that
\[
    \PP\Big( d_b\big(\dgm(\Filt(\M, d_{f,p})),\dgm(\Filt(\XX_n, d_{n,p}))\big) > \varepsilon \Big) \leq \exp{\left(-\theta n^{(1 - \lambda d) \alpha}\right)}
\]
for $n$ large enough and $\alpha = 1/(d+2p)$.
\end{corollary}

\subsection{Homology Inference}

The content of Corollary \ref{conv.diag} is that $\dgm(\Filt(\XX_n, d_{n,p}))$ is (asymptotically)  a good estimator of $\dgm(\Filt(\M, d_{f,p}))$. On the other hand, if we were to employ the Euclidean distance $|\cdot|$, it follows from the results in \cite{CGLM} that the sample persistence diagrams $\dgm(\Filt(\XX_n, | \cdot |))$ converge to $\dgm(\Filt(\M,|\cdot|))$ under reasonable hypotheses. We are therefore interested in comparing for how long we may expect to read the correct homology of $\M$ in each of the diagrams $\dgm(\Filt(\M, d_{n,p}))$ and $\dgm(\Filt(\M, | \cdot |))$ in terms of two natural geometric measures associated to the manifold, namely, the reach and the convexity radius  \cite[see][]{Hau, Lat, NSW, ChL}. In this section we show that the homology of $(\M, d_{f,p})$ can be recovered correctly from its persistence diagram up to the convexity radius $\convex(\M,d_{f,p})$, whereas for $(\M, |\cdot|)$ this is guaranteed only up to its reach $\tau_{\M}$. Notice that the reach of a submanifold of an Euclidean space depends strongly on the particular embedding, whereas the convexity radius is an intrinsic quantity linked to the geometry of the manifold. There are simple examples of manifolds in which this distinction is relevant to correctly  recover its homology from a sample (see Examples \ref{eyeglasses example} and \ref{ellipse}).
\par Recall that given $\XX \subseteq \RR^{D}$ a closed subset, the \textit{medial axis} $\mathrm{Med}(\XX)$ of $\XX$ is defined as
\[
    \mathrm{Med}(\XX) := \{y \in \RR^{D}: d_{E}(y,\XX) = |p-y| \text{ for at least two different points } p \in \RR^{D} \},
\]
where $d_{E}(y,\XX) = \inf _{x\in \XX}|y-x|$.
The \textit{reach} $\tau_{\XX}$ of $\XX$, first introduced in \cite{Fed},  is the minimum distance from $\XX$ to $\mathrm{Med}(\XX)$, that is,
\[
    \tau_{\XX} := \inf_{x \in \XX} d_{E}(x,\mathrm{Med}(\XX)).
\]
Given a Riemannian manifold $(\N,g)$, we will say that a subset $S \subseteq \N$ is  \textit{geodesically convex} if for every two points in $S$, there is a unique geodesic segment that connects them and it is completely contained in $S$. The \textit{convexity radius} $\convex(\N,x)$ at a point $x \in \N$ is the supremum over those $r > 0$ for which the (geodesic) ball $B(x,r)$ is geodesically convex. The convexity radius $\convex(\N)$ of the manifold $\N$ is
defined as
\[
    \convex(\N) := \inf_{x \in \N} \convex(\N,x).
\]

\begin{proposition} 
\label{thm homotopy reconstruction}
Let $\M$ be a compact submanifold of $\RR^D$. Then, we have the following homotopy equivalences:
\begin{itemize}
    \item $\Cech_{\epsilon}(\M,|\cdot|) \simeq \M$ for $\epsilon < \tau_{\M}$ and $\Rips_{\epsilon}(\M,|\cdot|) \simeq \M$ for $\epsilon < 2\sqrt{\frac{D+1}{2D}}\tau_{\M}$, and both bounds are optimal, in the sense that there exist examples for which the homotopy equivalence does not hold for larger values of $\epsilon$.
    \item $\Cech_{\epsilon}(\M,d_{f,p}) \simeq \M$ and $\Rips_{\epsilon}(\M,d_{f,p}) \simeq \M$ for $\epsilon < \convex(\M,d_{f,p})$.
\end{itemize}
Moreover, if $d_{f,p}$ coincides up to a constant with $d_{\M}$ (i.e. $f$ is uniform), we have the estimate
\[
    \convex(\M, d_{f,p}) = \Vol(\M,d_{\M})^{(p-1)/d} \convex(\M,d_{\M}) \geq \Vol(\M,d_{\M})^{(p-1)/d} \frac{\pi}{2} \tau_{\M}.
\]
\end{proposition}

\begin{proof} 
The fact that $\Cech_{\epsilon}(\M, |\cdot|)$ is homotopy equivalent to $\M$ for $\epsilon < \tau_{\M}$ is an immediate consequence of the Nerve Theorem. The same result implies that $\Cech_{\epsilon}(\M, d_{f,p}) \simeq \M$  for $\epsilon < \convex(\M, d_{f,p})$, since geodesically convex sets are always contractible and the intersection of geodesically convex sets is again geodesically convex. Regarding the Vietoris--Rips filtration, the fact that the simplicial complex $\Rips_{\epsilon}(\M,|\cdot|)$ is homotopy equivalent to $\M$ for $\epsilon < 2 \sqrt{\frac{D+1}{2D}} \tau_{\M}$ can be deduced from \cite[Theorem 20]{KSCRW}. Finally, since $d_{f,p}$ is a Riemannian distance on $\M$, there is an explicit homotopy equivalence $\Rips_{\epsilon}(\M, d_{f,p}) \simeq \M$ for $\epsilon < \convex(\M, d_{f,p})$ \cite[see][]{Hau, Lat}.
\par The optimality of the bound $\epsilon < \tau_{\M}$ for $\Cech_{\epsilon}(\M,|\cdot|)$ is clear (think of a unit sphere in $\RR^{D}$), and indeed, typically the topology of $\Cech_{\epsilon}(\M, |\cdot|)$ changes when $\epsilon$ attains $\tau_{\M}$. A critical example for the Vietoris--Rips complex is the standard 1-dimensional circle $\mathbb S^1$, and it can be derived from the main result of \cite{Adam}, similarly as in \cite[Example 24]{KSCRW}.
\par The last assertion in the statement follows directly from the inequalities
\[
    \convex(\M,d_{\M}) \geq \min \left\{ \frac{\pi}{2\sqrt{\sup K}}, \frac{1}{2}\inj(\M, d_{\M})\right\}
\]
\cite[see][\S 5.14]{ChEb} and
\[
    \inj(\M, d_{\M}) \geq \pi \tau_{\M} \text{,  } K \leq \frac{1}{\tau_{\M}^{2}}
\]
\cite[see][Proposition A.1]{Reach}. Here $\inj(\M,d_{\M})$ is the injectivity radius of $\M$ and $K$
is the sectional curvature.
\end{proof}

\begin{example} 
\label{ellipse}
Consider a planar ellipse $E_{R, \varepsilon}$ with minor axis of length $\varepsilon$ and major axis of length $R \geq \varepsilon$. By letting $R \to +\infty$ and/or $\varepsilon \to 0$, we see that the convexity radius of a closed submanifold of $\RR^{2}$ can be arbitrarily large while its reach can be arbitrarily small. A similar example can be constructed in $\RR^D$, being $\mathcal M$ a $d$-dimensional ellipsoid for any $d<D$. 
The same phenomenon can be achieved by constructing different \textit{eyeglasses} curves  with arbitrarily large length and constant reach, Figure \ref{fig:eyeglasses_reach}.
Its population persistence diagrams differ as predicted by Theorem \ref{thm homotopy reconstruction}.
The persistence diagram computed with the Euclidean distance captures the right homology only for $\epsilon$ less that the reach.  In contrast, for the Fermat distance the correct homology is captured for radii as large as (a multiple of) the convexity radius, which can be made large enough by enlarging the bridge between the glasses.

\begin{figure}[tb!]
    \centering

    \includegraphics[width =1.01 \textwidth]{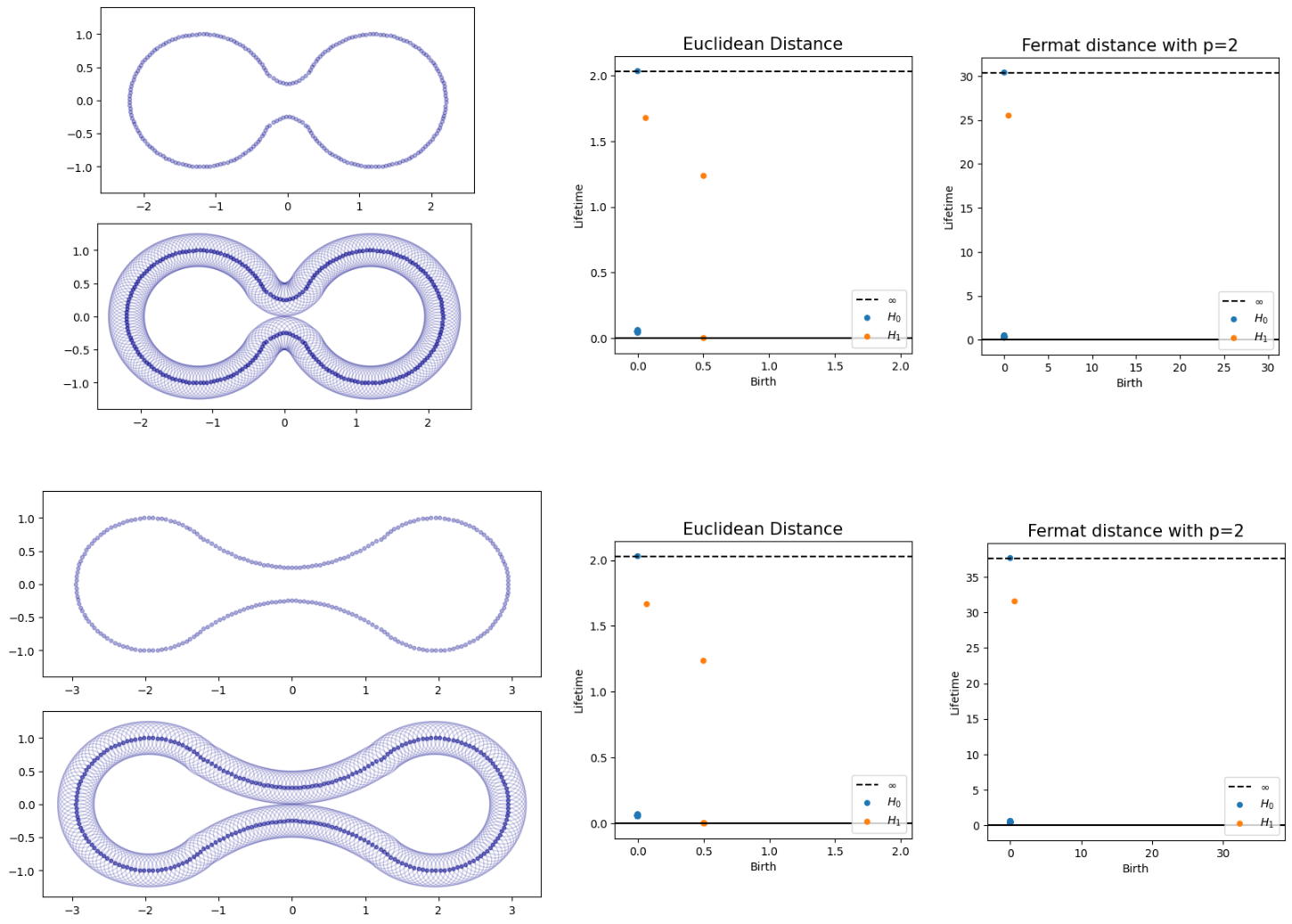}

    \caption{Left: Eyeglasses curves, uniformly sampled (250 points). In both cases, the reach is 0.5. Below each curve, we plot a thickening of the samples with Euclidean balls of radius slightly greater than the reach.  Right: Persistence diagrams (lifetime) associated to the Vietoris--Rips filtration for both the Euclidean distance and the re-scaled Fermat distance $d_{n,p}$ with $p=2$. 
    While $H_0$ is correctly estimated in both cases by reading the persistence diagrams,
    the ones computed with the Euclidean distance displays two salient generators for the first homology group $H_1$,
    inaccurately suggesting two cycles. The second cycle's birth is at the level of twice the reach. 
     For the (re-scaled) Fermat distance,  the diagrams shows correctly only one persistent generator for $H_1$. 
    }
    \label{fig:eyeglasses_reach}
\end{figure}
\end{example}

\subsection{Robustness to Outliers}\label{section outliers}

Persistence diagrams are highly sensitive to outliers \cite[see][]{MR2854318, chazal2011geometric, buchet2016efficient, MR3968644}.  We will see that the computation of persistence homology using Fermat distance is robust to the presence of outliers for positive degree.
Concretely, given a sample $\XX_n\subseteq \M$ and $Y\subseteq \RR^D\smallsetminus \M$ a finite set of points in the complement of $\M$ in the ambient Euclidean space --- the \textit{outliers} --- we prove that $\dgm_k(\Rips(\XX_n \cup Y, d_{\XX_n\cup Y, p}))$  coincides with $\dgm_k(\Rips(\XX_n, d_{\XX_n, p}))$ for $k>0$ up to some reasonable filtration parameter.
First we need a definition. 

\begin{definition}
Given a finite set of points $S\subseteq \RR^D$, define the \textit{minimal spacing} of $S$ as 
\[
\kappa(S) = \min_{x\in S} d_E(x, S\smallsetminus \{x\}),
\]
where $d_E$ denotes the Euclidean distance between sets.
\end{definition}

\begin{proposition}\label{outliers}
Let $\delta = \min \{\kappa(Y), d_E(\XX_n, Y)\}$ and $p>1$.
Then, for every $\epsilon <\delta^p$
\[
 \Rips_\epsilon(\XX_n \cup Y, d_{\XX_n\cup Y, p}) = \Rips_\epsilon(\XX_n, d_{\XX_n, p}) \cup Y.
 \]
 In particular, for all $k>0$
 \[
 \dgm_k(\Rips_{<\delta^p}(\XX_n \cup Y, d_{\XX_n\cup Y, p})) = \dgm_k(\Rips_{<\delta^p}(\XX_n, d_{\XX_n, p})),
\]
where $\Rips_{<\delta^p}(\XX, \rho_{\XX})$ stands for $\big(\Rips_{\epsilon}(\XX, \rho_{\XX})\big)_{\epsilon<\delta^ p}$, i.e., the Rips filtration  up to parameter $\delta^{p}$ of a metric space $(\XX, \rho_{\XX})$.
\end{proposition}

\begin{proof} 
Let us estimate the distance between two given points in $\XX_n\cup Y$ with respect to 
$d_{\XX_n\cup Y, p}$ in terms of $\delta$ and $d_{\XX_n,p}$.

If $x\in \XX_n$ and $y\in Y$,
\[
d_{\XX_n\cup Y, p}(x, y)\geq d_{\XX_n\cup Y, p}( \XX_n, Y)=d_E(\XX_n, Y)^p\geq \delta^p.
\]
\par If $y, y' \in Y$, 
\[
d_{\XX_n\cup Y, p}(y, y')\geq d_{\XX_n\cup Y, p}( y, Y\smallsetminus\{y\})\geq \delta^p.
\]
For the second inequality, notice that if $\tilde y\in Y$ is such that  $d_{\XX_n\cup Y, p}( y, Y\smallsetminus\{y\}) = d_{\XX_n\cup Y, p}(y, \tilde y) = \len(\gamma)$,  the geodesic $\gamma$ between $y$ and $\tilde y$  either involves only points from $Y$ or there exist some point $x\in \XX_n$ in $\gamma$.
In the first case $d_{\XX_n\cup Y, p}(y, \tilde y)\geq \kappa(Y)^p$ whereas in the second case $d_{\XX_n\cup Y, p}(y, \tilde y)\geq 2 d_E(\XX_n, Y)^ p$. 

Given $x, x' \in \XX_n$, let $\gamma$ be a  minimal path between $x, x'$, so that $d_{\XX_n\cup Y, p}(x, x') = \len(\gamma)$. If $d_{\XX_n\cup Y, p}(x, x')< \epsilon$, then $\gamma$ only involves points in $\XX_n$ since otherwise $\epsilon \geq \len(\gamma) \geq 2d_E(\XX_n, Y) \geq 2\delta^p$, which is a contradiction.
Hence, $d_{\XX_n\cup Y, p}(x, x') = d_{\XX_n, p}(x, x')$.
\end{proof}

We define now a geometric notion of outliers. Recall that given $\XX_n\subseteq \RR^D$, the $\varepsilon$-graph $G_\varepsilon (\XX_n)$ is the undirected graph with the points of $\XX_n$ as vertices and an edge connecting $x_i$ and $x_j\in \XX_n$ whenever $|x_i-x_j| <\varepsilon$.

\begin{definition}
Let $\XX_n\subseteq \M$ be a sample of $\M\subseteq \RR^D$ and $Y\subseteq \RR^D\smallsetminus \M$ be a finite set of points.
 Let $\varepsilon_* := \min\{\varepsilon>0: G_\varepsilon (\XX_n)\text{ is connected}\}$ and $\delta = \min \{\kappa(Y), d_E(\XX_n, Y)\}$.
 We say that $Y$ are \textit{(geometric) outliers} if $\delta >\varepsilon_*$.
\end{definition}

We show next that for this notion of outliers, the upper bound on the parameter for the Rips filtration of Proposition \ref{outliers} is not restrictive for sufficiently large $p$. Indeed, let $\diam_p(\XX_n)$ be the diameter of $(\XX_n, d_{\XX_n,p})$. Note that for every $\epsilon \geq \diam_p(\XX_n)$ the simplicial complex $\Rips_{\epsilon}(\XX_n, d_{\XX_n, p})$ equals the standard $(n-1)$-simplex $\Delta^ {n-1}$, with trivial topology (and hence persistence diagrams are not interesting for scales larger than this threshold). The next result states that provided that $p$ is large enough, the persistence diagrams of $(\XX_n, d_{\XX_n, p})$ and $(\XX_n \cup Y, d_{\XX_n\cup Y, p})$ coincide up to the filtration parameter $\diam_p(\XX_n)$.

\begin{corollary} \label{threshold}
Given $\XX_n$ a sample of $\M$ and $Y\subseteq \RR^D$ a finite set of outliers, then 
for all $k>0$
 \[
 \displaystyle \dgm_k(\Rips_{<\diam_p(\XX_n)}(\XX_n \cup Y, d_{\XX_n\cup Y, p})) = \dgm_k(\Rips_{<\diam_p(\XX_n)}(\XX_n, d_{\XX_n, p})).
\]
for $p>C\log(n)$ with $C=\log(\delta/\epsilon_*)^{-1}$.
 \end{corollary}
 \begin{proof} 
There is an upper bound $\diam_p(\XX_n) \leq n\varepsilon_*^ p$. Since $Y$ are outliers, $\varepsilon_*<\delta$ . For $p>C\log(n)$, $\left(\frac{\delta}{\varepsilon_*}\right)^p>n$ and consequently, $\diam_p(\XX_n)<\delta ^p$. The result now follows from Proposition \ref{outliers}.
 \end{proof}

\begin{remark} \label{H0} In general, the persistence diagram of $(\XX_n \cup Y, d_{\XX_n\cup Y, p})$ for degree $k=0$ does not coincide with the diagram of the metric space without outliers $(\XX_n , d_{\XX_n, p})$. However, if $Y$ is a set of geometric outliers, it is related to the corresponding persistence diagrams of $\XX_n$ and $Y$ through the following formula:
\[\dgm_0(\Rips(\XX_n \cup Y, d_{\XX_n\cup Y, p}))=\dgm^{<\infty}_0(\Rips(\XX_n, d_{\XX_n, p})) \cup \dgm_0(\Rips(Q, d_Q)).
\]
Here, $\dgm^{<\infty}$ denotes the bounded persistence intervals and $Q=(Y\cup \XX_n)/\XX_n$ is the quotient metric space endowed with the induced metric  $d_Q$.
\end{remark}

\begin{remark}[DTM] \label{DTM} Filtrations classically used for the computation of persistent homology of Euclidean point clouds, such as the \v{C}ech or Vietoris--Rips filtrations, are very sensitive to the presence of outliers. That is,  \v{C}ech (or Vietoris--Rips) filtrations computed on top of $\XX_n$ and $\XX_n\cup Y$ might be very different (its interleaving distance depends on $d_H(\XX_n, \XX_n\cup Y)$, see e.g. \cite{CDSO}).  To overcome this limitation,
\cite{MR3968644} introduced weighted filtrations based on the notion of distance to measure (DTM).
Given $\mu$ the empirical measure of $\XX_n\subseteq \RR^D$ and $m\in[0,1)$ a parameter, the DTM-function over $\RR^D$ is   defined as $d_{\mu, m}(x):=\sqrt{\frac{1}{m}\int_0^m\delta^2_{\mu,t}(x)dt}$,
where $\delta_{\mu,t}(x) = \inf\{r \geq 0\colon \mu(\bar B(x, r)) > t\}$ and $\bar B(x, r)$ denotes the closed Euclidean ball with center $x$ and radius $r$. Given a parameter $p>1$,  the weighted ball $B_{d_{\mu, m}}(x,\epsilon)$ with center $x\in\XX_n$ and radius $\epsilon\geq d_{\mu, m}(x)$ is the Euclidean ball $B(x,r_x(\epsilon))$ with radius $r_x(\epsilon) =\left(\epsilon^p-d^p_{\mu, m}(x)\right)^{1/p}$ (if $\epsilon< d_{\mu, m}(x)$, it is empty).
The \v{C}ech DTM-filtration $(V^{DTM}_{m,p}(\XX_n))_{\epsilon>0}$ with parameters $(m, p)$ is the weighted \v{C}ech filtration constructed as the nerve of the cover $\{B_{d_{\mu, m}}(x,\epsilon):x\in \XX_n\}$ for every $\epsilon>0$. A DTM-based version of a weighted Vietoris--Rips filtration can also be derived.

DTM-filtrations of Euclidean point clouds produce filtrations (and hence, persistence diagrams) less sensitive to outliers, given that the (interleaving) distance between $V^{DTM}_{m,p}(\XX_n)$ and $V^{DTM}_{m,p}(\XX_n\cup Y)$ is upper bounded not only in terms of $d_H(\XX_n, \XX_n\cup Y)$ but also in terms of the Wasserstein distance between the measures $\mu_{\XX_n}$ and $\mu_{\XX_n\cup Y}$.
However, if $\XX_n$ is a sample of a manifold $\M$, these filtrations are still very sensitive to the particular embedding of the manifold in $\RR^D$. This is consequence of the dependence of the DTM-function on the
ambient space (see Example \ref{ex:trefoil}). Its (lack of) dependence on non-intrinsic properties has been investigated thereafter. In this direction, a generalization of DTM-filtrations for general metric spaces $(\XX, \rho)$ is considered in \citep{buchet2016efficient}.
\end{remark}

\begin{example}[Trefoil] \label{ex:trefoil}
\label{trefoil} Consider the embedding of a topological circle $\mathbb S^1$ in $\RR^3$ given by the \textit{trefoil knot}. In particular, it is homeomorphic to $\mathbb S^1$ and its homology has just one generator in $H_0$ (one connected component) and one generator in $H_1$ (one 1-dimensional cycle).
Given a (noisy) sample of 1500 points from the trefoil knot with  10 outliers, Figure \ref{fig:trefoil}, we compute its persistence diagram for different choices of filtrations and compare them with the case without the outliers, Figure \ref{fig:PH_trefoil}. For the Vietoris--Rips filtration using Euclidean distance, the small reach of the embedding produces a persistence diagram with four persistent generators for $H_1$ in both cases, with and without outliers (cf. Example \ref{ellipse}). If we use $k$-NN distances, the presence of outliers affects the accuracy of the topological features captured in the persistence diagram, which presents four salient generators for $H_1$ instead of the single generator recovered from the sample without outliers.
For the Vietoris--Rips DTM-filtration, we observe that the diagrams are comparable both in absence and presence of outliers. However, the dependence of the embedding of the construction is reflected in the incorrect number of generators for $H_1$ with long persistence.
Finally, the persistence diagram computed from the Vietoris--Rips filtration using Fermat distance remains unaffected in presence of outliers for degree 1 (Corollary \ref{threshold}), and it shows correctly a single salient generator of $H_1$. For degree 0, the diagram is related to  the diagram of the sample without the outliers and the diagram of the outliers themselves (cf. Remark \ref{H0}).

\begin{figure}[tb!]
    \centering
\includegraphics[width=0.38\textwidth]{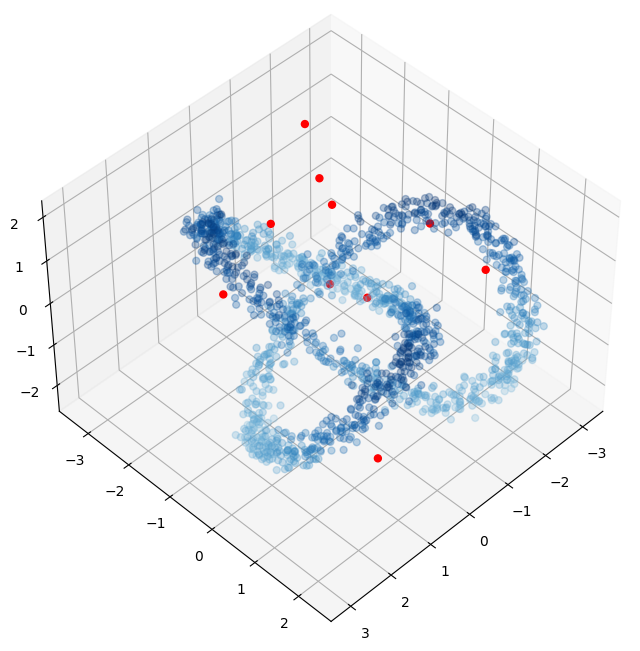}
\includegraphics[width=0.45\textwidth]{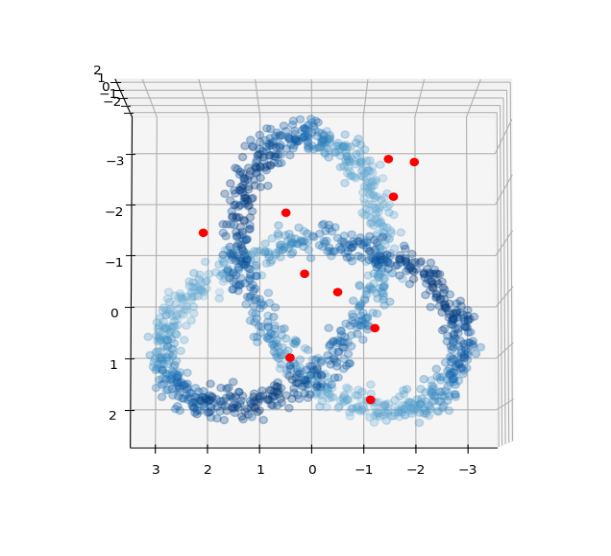}
\caption{A (noisy) sample of 1500 points from the trefoil knot with outliers (red).}
\label{fig:trefoil}
\end{figure}

\begin{figure}[ptb!]
\begin{minipage}{0.272\textwidth}

\includegraphics[width=1\textwidth]{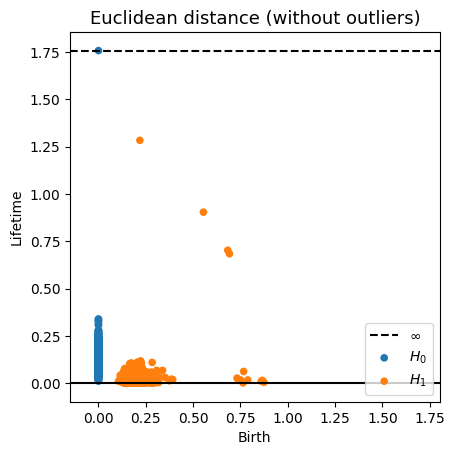}

\includegraphics[width=1\textwidth]{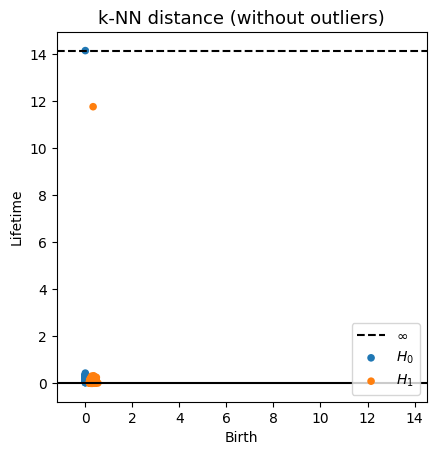}

\includegraphics[width=1\textwidth]{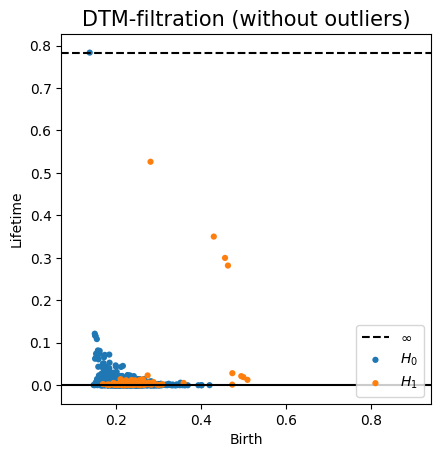}

\includegraphics[width=1.03\textwidth]{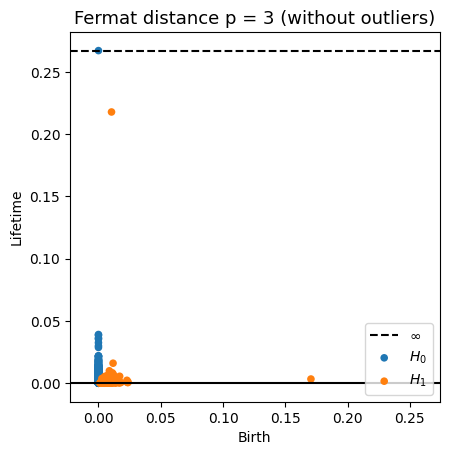}
\end{minipage}
\begin{minipage}{0.72\textwidth}
\centering
 \includegraphics[width=0.38\textwidth]{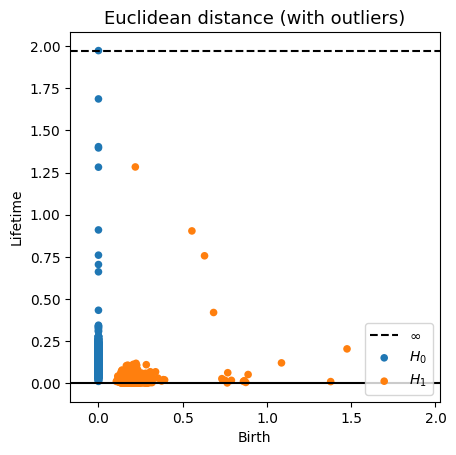}

\includegraphics[width=0.367\textwidth]{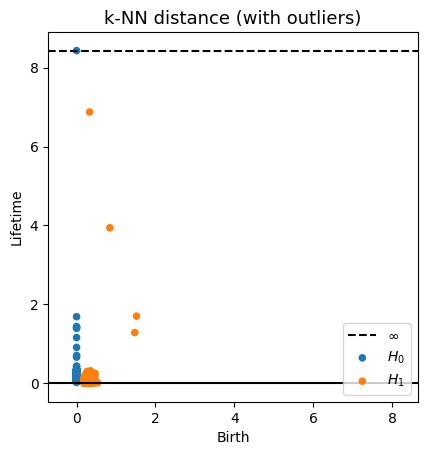}

\includegraphics[width=0.38\textwidth]{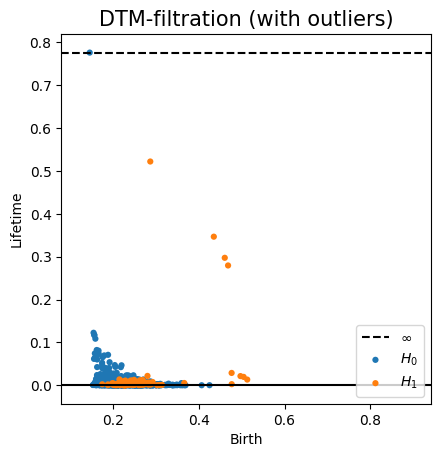}

\includegraphics[width=0.448\textwidth]{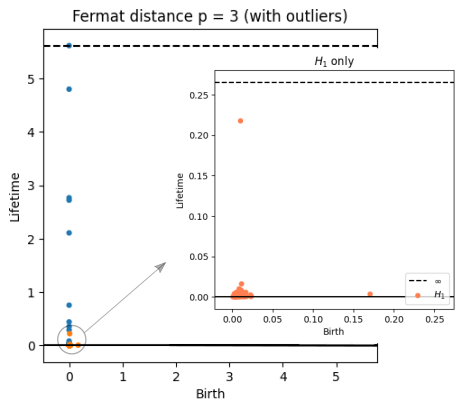}
\includegraphics[width=0.539\textwidth]{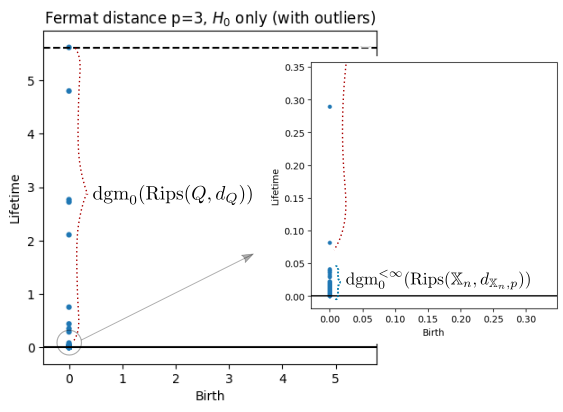}
\end{minipage}

\caption{Persistence diagrams associated to the Vietoris-Rips filtration of the sample of the trefoil knot using Euclidean distance, $k$-NN distance with $k=10$, DTM weight
and Fermat distance with $p=3$ of the sample without outliers $\XX_n$ (left) and the sample with outliers $\XX_n \cup Y$ (right) respectively.
When Fermat distance is used, the persistence diagram of $\XX_n\cup Y$ for degree 1 
equals the diagram of $\XX_n$ (without outliers). For degree 0, it decomposes as the union of the subdiagram of finite intervals of $\XX_n$, $\dgm^{<\infty}_0(\Rips(\XX_n, d_{\XX_n, p}))$, and the diagram $\dgm_0(\Rips(Q, d_Q))$ of the quotient space
 $Q=(Y\cup \XX_n)/\XX_n$.
}
\label{fig:PH_trefoil}
\end{figure}

\end{example}

\subsection{Computational Complexity}

Our proposed pipeline for the computation of Fermat-based persistent homology consists of the precomputation of Fermat distance in the input sample $\XX_n$, followed by the computation of persistent homology from the metric space $(\XX_n, d_{\XX_n, p})$ described by the distance matrix.

The computation of the matrix of pairwise sample Fermat distances between points in $\XX_n$
has complexity $\O(n^3)$. However, it can be reduced to $\O(n^2\log^2 n)$ with high probability by restricting the computation of shortest paths to the $k$-NN graph on top of $\XX_n$ with $k = O(\log n)$ (see Section 2.3 in \cite{GJS}, also \cite{LMM, CMS}).

On the other hand, the \textit{standard algorithm} used to compute persistent homology was first introduced in \cite{ELZ} and it is based on the Gaussian reduction of the boundary matrix. Persistent homology for degree up to $k$ depends on the $(k+1)$-skeleton of the filtration and the worst case computational complexity is cubical in the number $N$ of simplices of dimension at most $k+1$ \citep{morozov2005persistence, otter2017roadmap}. An alternative algorithm for the reduction of the boundary matrix, introduced in \cite{MMS}, has complexity $O(N^{\omega})$, with $\omega$ the matrix multiplication coefficient. At present, the best bound for $\omega$ is 2.376 \citep{coppersmith1987matrix}.

In practice, computation of persistent homology has lower complexity. For Vietoris--Rips filtrations, the worst case complexity is for $k$-dimensional persistent homology is $O\left({{n}\choose{k+2}}^3\right) = O\left(n^{3(k+2)}\right)$ with $n$ the number of vertices of $\XX_n$. However, 
in \cite{GHK22} it proved that, for instance, the average complexity for the reduction of the boundary matrix of degree 1 is upper bounded by $O(n^5\log^2(n))$.
Moreover, they showed that this upper bound seems to be not tight, since experimental simulations show that the average cost of the reduction of the 1-boundary matrix follows a curve of around $O(n^{3.73})$.

Overall, our proposed pipeline based on the precomputation of pairwise Fermat distance in $\XX_n$ does not increase the complexity of the total persistent homology computation.

\section{Applications to Signal Analysis}\label{experiments}

In this section we present a method for change-point detection and pattern recognition in time series through the analysis of topological features (see also \cite{MR3501790, P, PH}). This method is illustrated by a series of experiments in both synthetic and 
real data. In the experiments, the use of Fermat distance (as opposed to Euclidean distance) is observed to lead to more robust inference of the topology of the underlying space. We remark that in these examples the data does not necessarily verify the i.i.d. assumption.

 Fermat and  $k$-NN distances are computed using the library {\fontfamily{lmss}\selectfont Fermat} \citep{fermat_package}, while {\fontfamily{lmss}\selectfont Ripser}  \cite{ripser2021} is employed for the computation of persistence diagrams associated to Vietoris--Rips filtrations. All the computations are over the field $\mathbf{k}=\ZZ_2$. The code for all the examples and experiments can be found in the repository \cite{repository}.

\subsection{Topological Analysis of Time Series.}
Time-delay embeddings of scalar time-series data is a well-known technique to recover the underlying dynamics of a system. Takens' theorem \cite{MR654900} gives conditions under which a smooth attractor can be reconstructed from a generic observable function, with dimensional bounds related to those of the Whitney Embedding Theorem. It implies in particular that if $X(t)$ is a real valued signal (which is assumed to be one of the coordinates of a flow given by a system of differential equations), then the \textit{delay coordinate map}
\[
t\mapsto \Big(X(t), X(t+\tau), X(t+2\tau) \dots, X(t+(D-1)\tau)\Big)
\]
is an embedding of an orbit. Here $D$ is the embedding dimension and $\tau$ is the time delay.
From a theoretical point of view, $D$ is the number of variables of the original system.
However, in practice the underlying equations describing the dynamical system are not available.
Thus, dynamics are often analyzed by studying the topology of their \textit{attractors}; i.e., invariant subsets of the phase space towards which the system tends to evolve \citep{BW, MR228014, GL}.
 If the attractor is a smooth manifold $\M$ of dimension $d$, under certain conditions Takens' theorem  implies that the delay embedding of the signal with $D\geq 2d+1$ is  diffeomorphic to $\M$.

We describe now an approach --- based on intrinsic persistence diagrams --- to study geometry of attractors and pattern recognition in time series by means of the analysis of the time evolving topological organization of the embedded flow. 
Let $(x_1, x_2, \dots, x_n)$ be a time series, i.e. a finite sample of a signal $X:[0,T]\to \RR$ such that for evenly spaced points $0=t_1<t_2<\dots<t_n=T$, $x_i = X(t_i)$ for all $1\leq i\leq n$. Given $D$ and $\tau$, compute the delay embedding of the time series 
\[\XX_n = \{(x_i, x_{i+\tau},  x_{i+2\tau}, \dots, x_{i+(D-1)\tau}): 1\leq i \leq n-(D-1)\tau\}\subseteq \RR^D.\]
Then, for $p>1$, endow $\XX_n$ with a metric space structure induced by the sample Fermat distance $d_{\XX_n, p}$. 
The persistence diagram of the delay embedding $(\XX_n, d_{\XX_n, p})$ quantifies information about the homology of the attractor associated to the underlying dynamical system.

\begin{example} [Reconstruction of Lorenz attractor]

The parameters associated to the delay coordinate reconstruction for a time series can be determined following some heuristics (e.g. \textit{false nearest neighbors} to determine the embedding dimension \citep{embedding_dim}). However, in case of noisy data, the embedding dimension is often over-estimated and it may have a great impact on the phase space reconstruction. Indeed, in high dimensional spaces, any two points of a typical large set are at similar Euclidean distance \citep{aggarwal2000on}. This phenomenon is part of what is known as the \textit{curse of dimensionality}. For this reason, the choice of an intrinsic distance is crucial to recover the right topological features of a space embedded in high dimension.

Consider the strange attractor associated to the   Lorenz system \cite{MR4021434}
\begin{equation}\label{eqn:lorenz}
\begin{cases}
\dot x = \sigma ( y - x ) , \\
\dot y = x(\rho-z)-y ,\\
\dot z = xy-\beta z \end{cases}
\end{equation}
when $(\sigma, \rho, \beta) = (10, 28, 8/3)$. 

In Figure \ref{fig:lorenz} we take a numerical integration $\varphi(t, v_0)$ of \eqref{eqn:lorenz} with $dt = 0.01$, satisfying the initial condition $\varphi(0,v_0)=v_0$ with $v_0=(1,1,1)$.
We inspect the time series corresponding to the $x$-coordinate with additive Gaussian noise with variance $0.1$, and recover topological information of the attractor from the delay embedding (see also \cite{MR3501790}). Notice that in this case, although the number of variables in the underlying system is 3, the dimension of the attractor is $d=2$ so the  embedding dimension estimated by Takens' theorem is greater than or equal to 5.

The persistence diagram of the delay embedding reconstruction is computed with time delay $\tau=10$ and embedding dimensions $D=3, 4$ and $5$, Figure \ref{fig:lorenz}. Here, a uniform down-sampling from the original point cloud of $\sim 10000$ points is computed, to obtain a new point cloud of $\sim 3400$ points.

\begin{figure}[ptb!]
    \centering
    \includegraphics[width=0.99\textwidth]{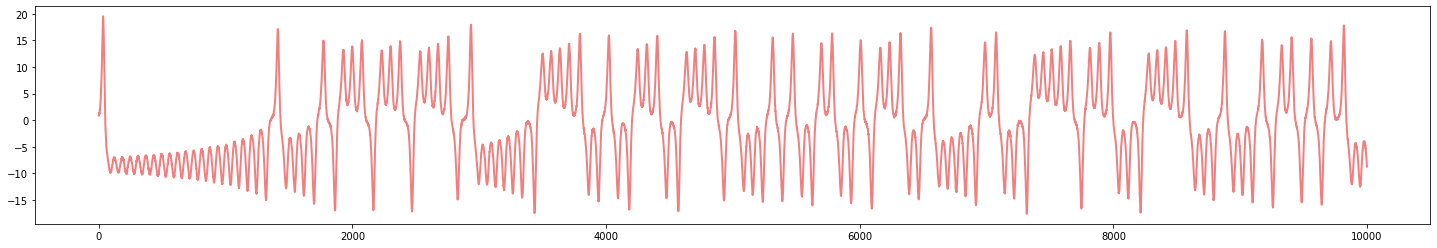} 
    
    \includegraphics[width=0.38\textwidth]{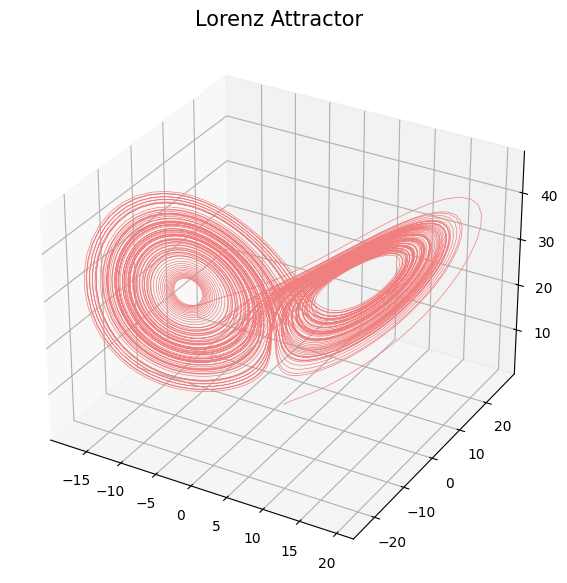}
    \includegraphics[width=0.4\textwidth]{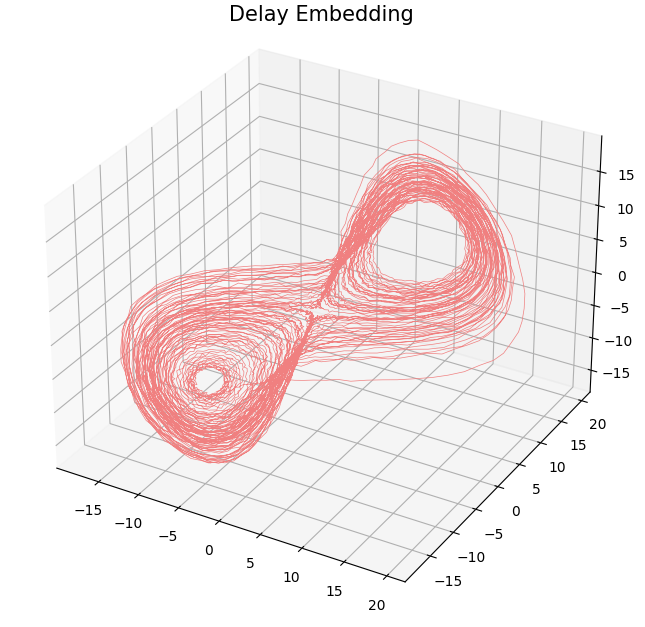}
    
    \includegraphics[width=0.293\textwidth]{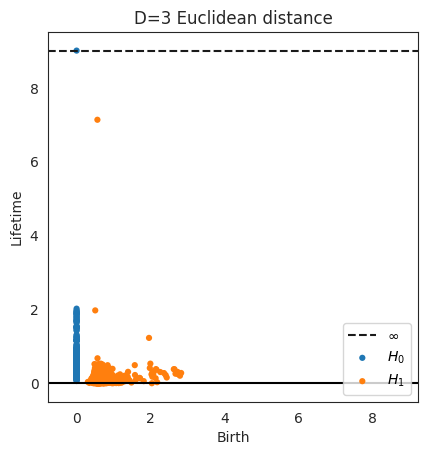}
    \includegraphics[width=0.3\textwidth]{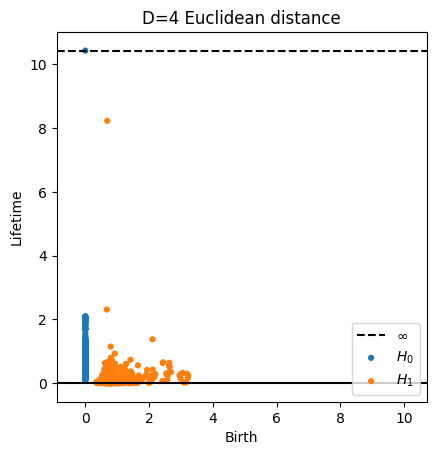}
    \includegraphics[width=0.3\textwidth]{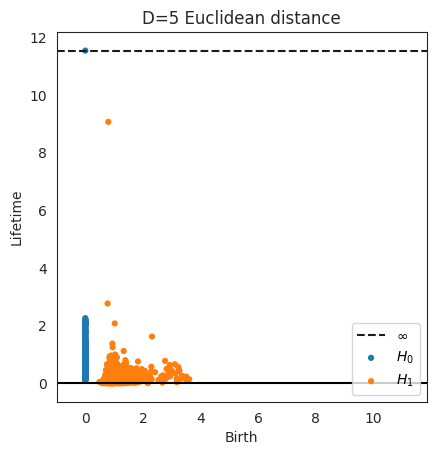}

    \includegraphics[width=0.3\textwidth]{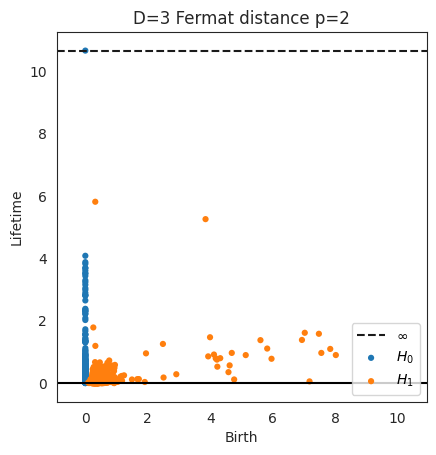}
    \includegraphics[width=0.3\textwidth]{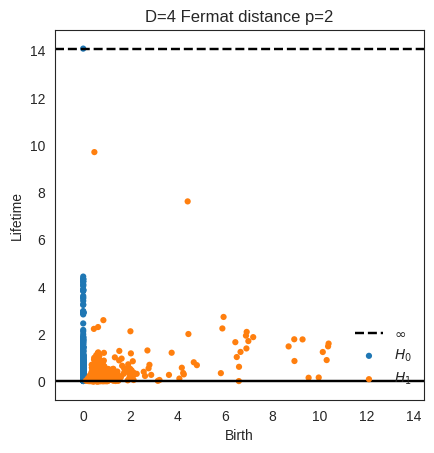}
    \includegraphics[width=0.3\textwidth]{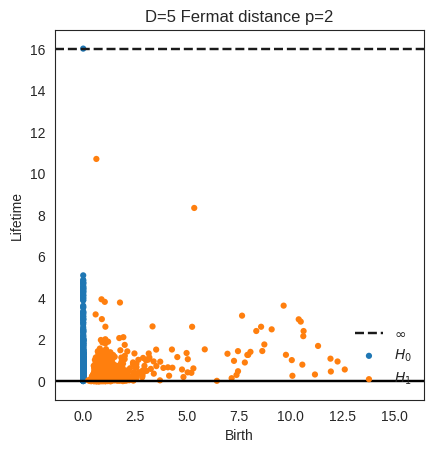}
    \caption{From top to bottom: The $x$-coordinate time series with Gaussian noise (variance = 0.1) of the Lorenz attractor. The original trajectory and the delay embedding of the noisy $x$-coordinate time series with $D=3$ and $\tau = 10$. Persistence diagrams associated to the delay embedding computed with Euclidean and Fermat distances for embedding dimension $D=3$, $D=4$ and $D=5$ and time delay $\tau=10$.}
    \label{fig:lorenz}
 \end{figure}

The Lorenz attractor is homotopy equivalent to the \textit{eight-space} with two holes corresponding to the equilibrium points that the trajectory never reaches.
As Figure \ref{fig:lorenz} reveals, the use of Fermat distance leads to robustly capturing the intrinsic two 1-cycles for the different embedding dimensions, while this is not the case for the Euclidean distance.

\end{example}

\begin{example} [Periodicity] A periodic dynamic within a noisy system might be robustly captured using time-delay embeddings. Indeed, embeddings of periodic signals have the topology of a cycle. However, the general success of the reconstruction of the intrinsic cyclic geometry is highly dependent on the choice of the delay parameter $\tau$ (and the embedding dimension $D$). In practice, classic heuristics based on time-delayed mutual information \citep{embedding_delay} and false nearest neighbors \citep{embedding_dim} are used, but they present high sensitiveness to noise. We  show that the use of Fermat distance when  recovering the intrinsic geometry of delay embeddings has stability properties with respect to the choice of $\tau$. 

\par Consider the function $f(t) = \cos(t) + \cos(3t)$ with additive Gaussian noise of variance 0.4. For a sample of 2000 points of the noisy signal in consideration at the interval $[0,100]$, the classic heuristic estimations of the optimal parameters outputs $\tau = 28$ and $D=8$ (here, the computations are preformed with the package {\fontfamily{lmss}\selectfont Time Series} from the software {\fontfamily{lmss}\selectfont Giotto-tda} \cite{tauzin2020giottotda}). However, the associated time-delay embedding presents low reach value and, hence, it is still hard to capture its  homology with standard methods (see Figure \ref{fig:periodicity}).

\par In general dynamics, the effect of the choice of $\tau$  is reflected in changes in the embedding of the associated attractor in the ambient space. Although Takens' theorem theoretically establishes diffeomorphic embeddings for different choices of $\tau$, in practice the accuracy of the reconstruction of the underlying manifold usually depends on the choice of $\tau$. Crucially, persistence diagrams computed using Fermat distance are less dependent of extrinsic properties and hence,  highly appropriate for the estimation of topological properties of the attractor (that are, indeed, independent of the embedding).
To illustrate the stability properties with respect to the choice of the delay parameter, we computed the delay embedding of the noisy periodic signal of Figure \ref{fig:periodicity} in $\RR^8$ for a range of values of $\tau$. We observe that, while the features displayed on the diagrams computed using Euclidean distance change with the embedding, the ones computed using Fermat distance are consistent: they all display a single generator for $H_1$ (Figure \ref{fig:stability_tau}). Here, $p$ was set equal to  6, but similar results can be obtained for a range of values of $p$. 

\begin{figure}[tb!]
\centering
    \includegraphics[width=0.99\textwidth]{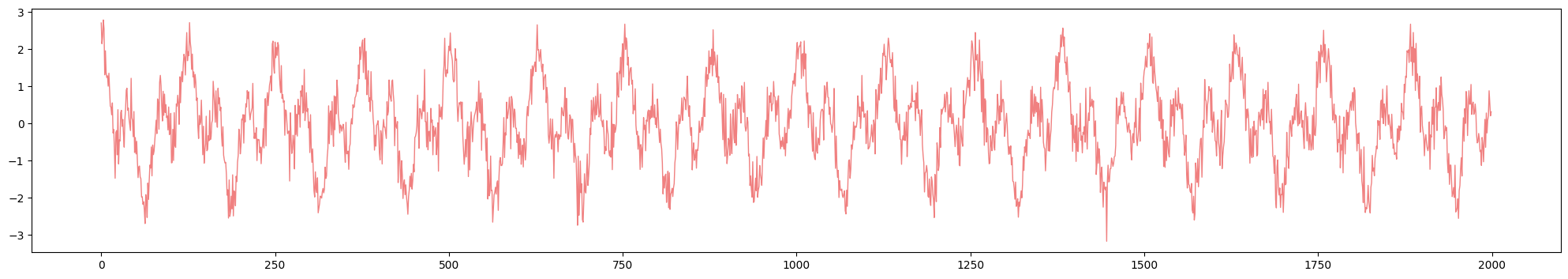} 
    
    \vspace{5pt}
    
    \includegraphics[width=0.95\textwidth]{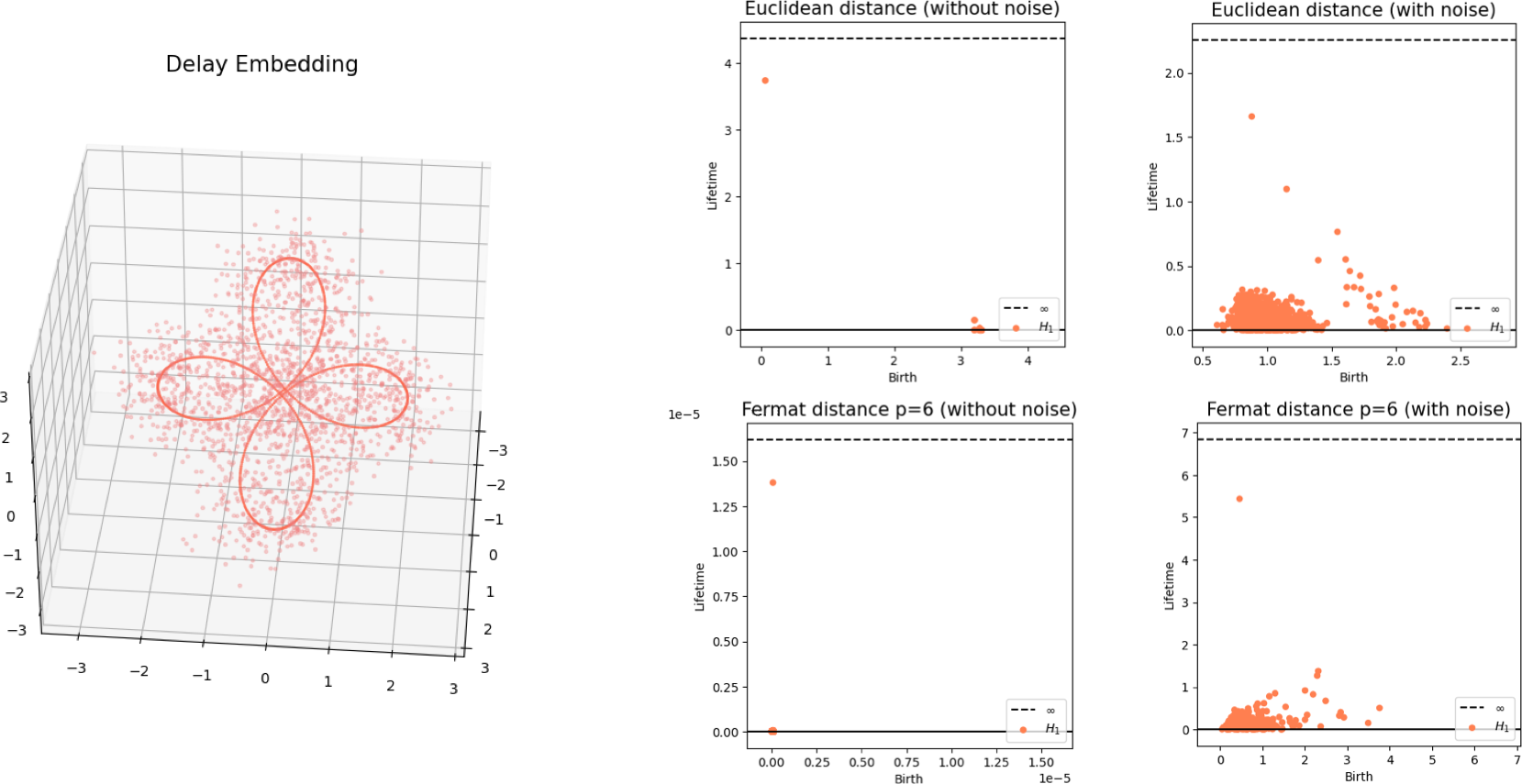}
    
    \caption{Top: Periodic signal with noise, defined as $f(t) = \cos(t) + \cos(3t)$ with additive Gaussian noise of variance $0.4$. Bottom left: Delay embedding (projection 3d to the first coordinates) with the optimal values of the parameters, i.e $D=8, \tau = 28$, according to the canonical heuristics (embedding of the signal without noise in dark orange).  Bottom right: Persistence diagrams (degree 1 only) of the embedding of the signal without and with noise, computed using the Euclidean distance and Fermat distance for $p=6$.}
    \label{fig:periodicity}
\end{figure}

\begin{figure}[tb!]
    \centering

    \includegraphics[width=\textwidth]{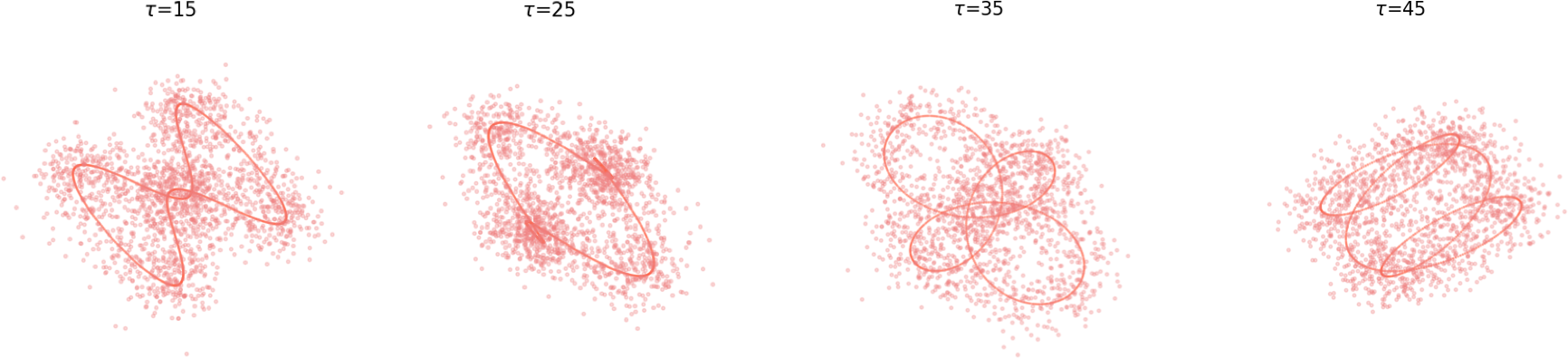}

    \vspace{10pt}
    
    \includegraphics[width=0.24\textwidth]{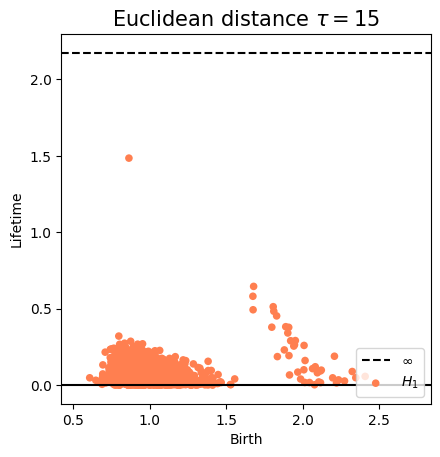}
    \includegraphics[width=0.24\textwidth]{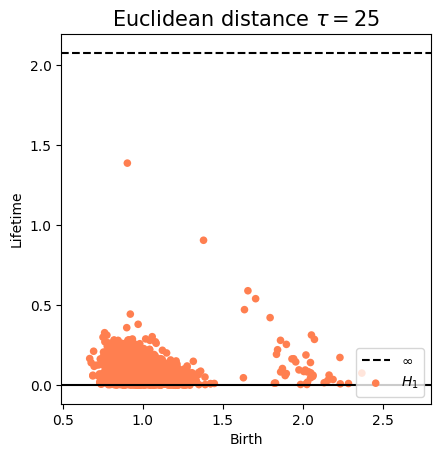}
    \includegraphics[width=0.24\textwidth]{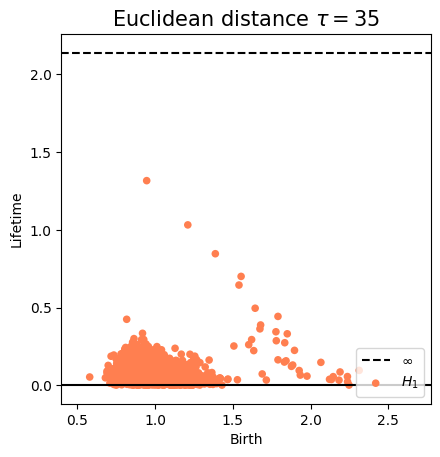}
    \includegraphics[width=0.24\textwidth]{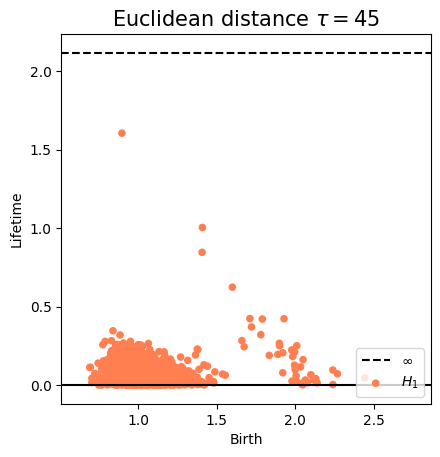}
    
    \vspace{5pt}
    
    \includegraphics[width=0.24\textwidth]{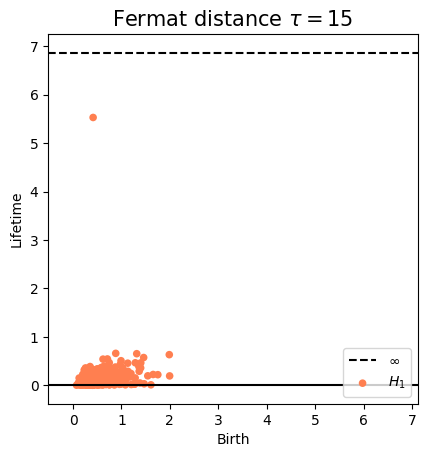}
    \includegraphics[width=0.24\textwidth]{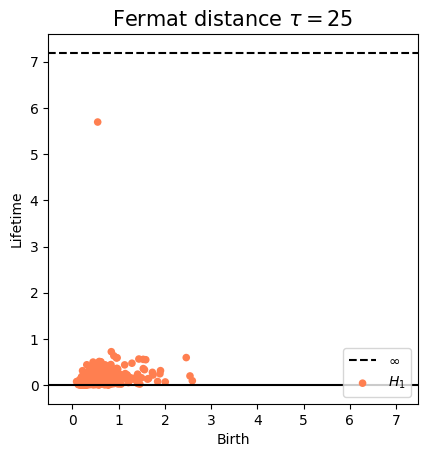}
    \includegraphics[width=0.24\textwidth]{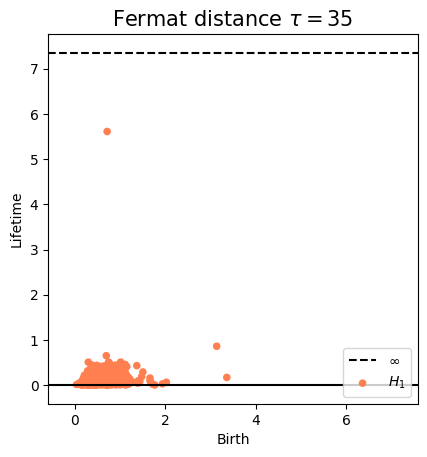}
    \includegraphics[width=0.24 \textwidth]{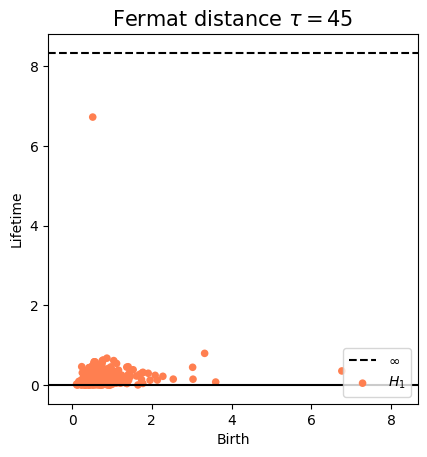}
    
    \caption{Top: Time-delay embeddings in $\RR^8$ (projection 3d to the first coordinates) for $\tau = 15, 25, 35, 45$ of the signal $f(t) = \cos(t) + \cos(3t)$ with additive Gaussian noise of variance $0.4$ (cf. Fig. \ref{fig:periodicity}). Bottom: Persistence diagrams (degree 1 only) using Euclidean distance and Fermat distance (for $p=6$, but similar outputs are obtained for a range of values of $p$).}
    \label{fig:stability_tau}
\end{figure}
\end{example}

In order to identify changes in patterns of time series, we investigate the topological evolution in time of the delay embedding.
For every sample time $t_j\in [0,T]$ ($1\leq j \leq n-(D-1)\tau)$, consider the delay embedding $\XX_j$ of the restriction of the time series up to time $t_j$, with the  metric structure inherited from $(\XX_n, d_{\XX_n, p}$). That is,
\[\XX_j :=  \{(x_i, x_{i+\tau},  x_{i+2\tau}, \dots, x_{i+(D-1)\tau}): 1\leq i \leq j\}\subseteq \XX_n.\]
If $\M[0,t]$ is the delay embedding of the restricted signal $X|_{[0,t]}$, the time evolving series of diagrams $\{\dgm(\Rips(\XX_i)):1\leq j \leq n-(D-1)\tau\}$ is a sample of an approximation of the curve 
\begin{equation}\label{curve diagrams} t\mapsto \dgm(\Rips(\M[0,t])),
\end{equation}
where $\M[0,t]$ is considered a metric subspace of $\M = \M[0,T]$ endowed with the population Fermat distance.
Finally, compute  
\begin{equation}\label{first derivative}\dfrac{d_b\big(\dgm(\Rips(\XX_i)), \dgm(\Rips(\XX_{i-1}))\big)}{t_{i}-t_{i-1}}
\end{equation}
as an approximate the `first order derivative' of \eqref{curve diagrams}. Shifts in patterns in the signal can be detected from the sample as peaks in the bottleneck distance between consecutive persistence diagrams.

Some applications of this technique follow below.

\begin{example}[Anomaly detection in ECG] The purpose of this example is to present a computational method of automated detection of abnormal heartbeats (arrhythmia) through the topological analysis of a delay embedding of ECG signals. We consider the 
record \textit{sel102} of the \textit{QT Database} from the freely-available repository of medical research data PhysioNet \cite{physionet_database}, Figure \ref{fig:ECG}. 

\begin{figure}[htb!]
\centering
\includegraphics[width=1\textwidth]{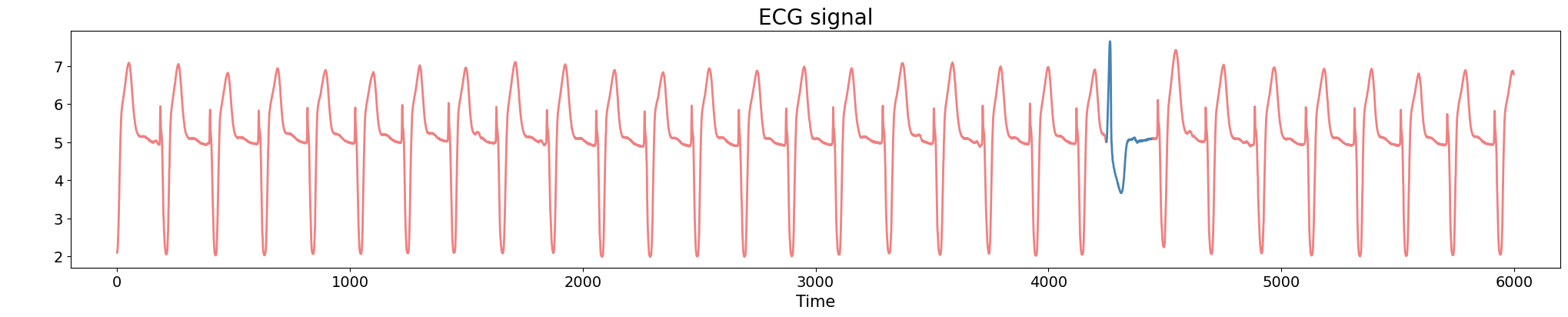} 
 \includegraphics[width=0.995\textwidth]{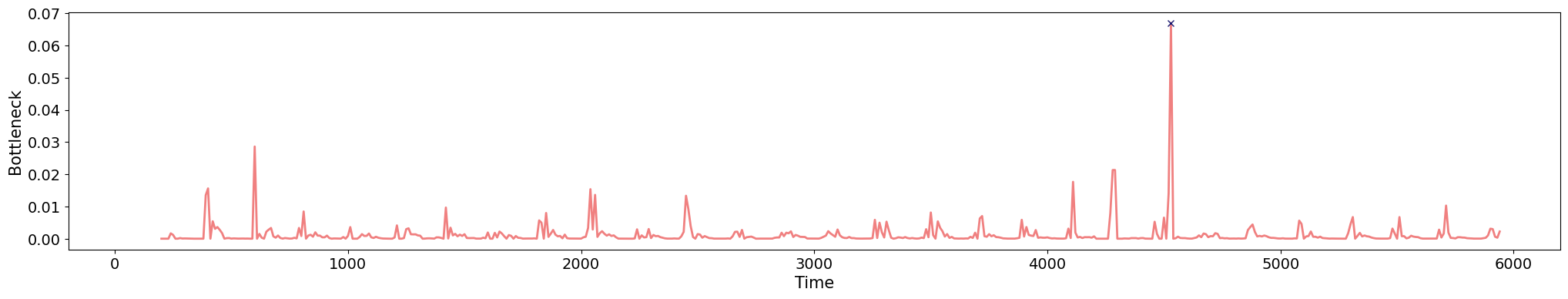} 

\begin{minipage}{0.5\textwidth}
\includegraphics[width=\textwidth, ,height=5.6cm]{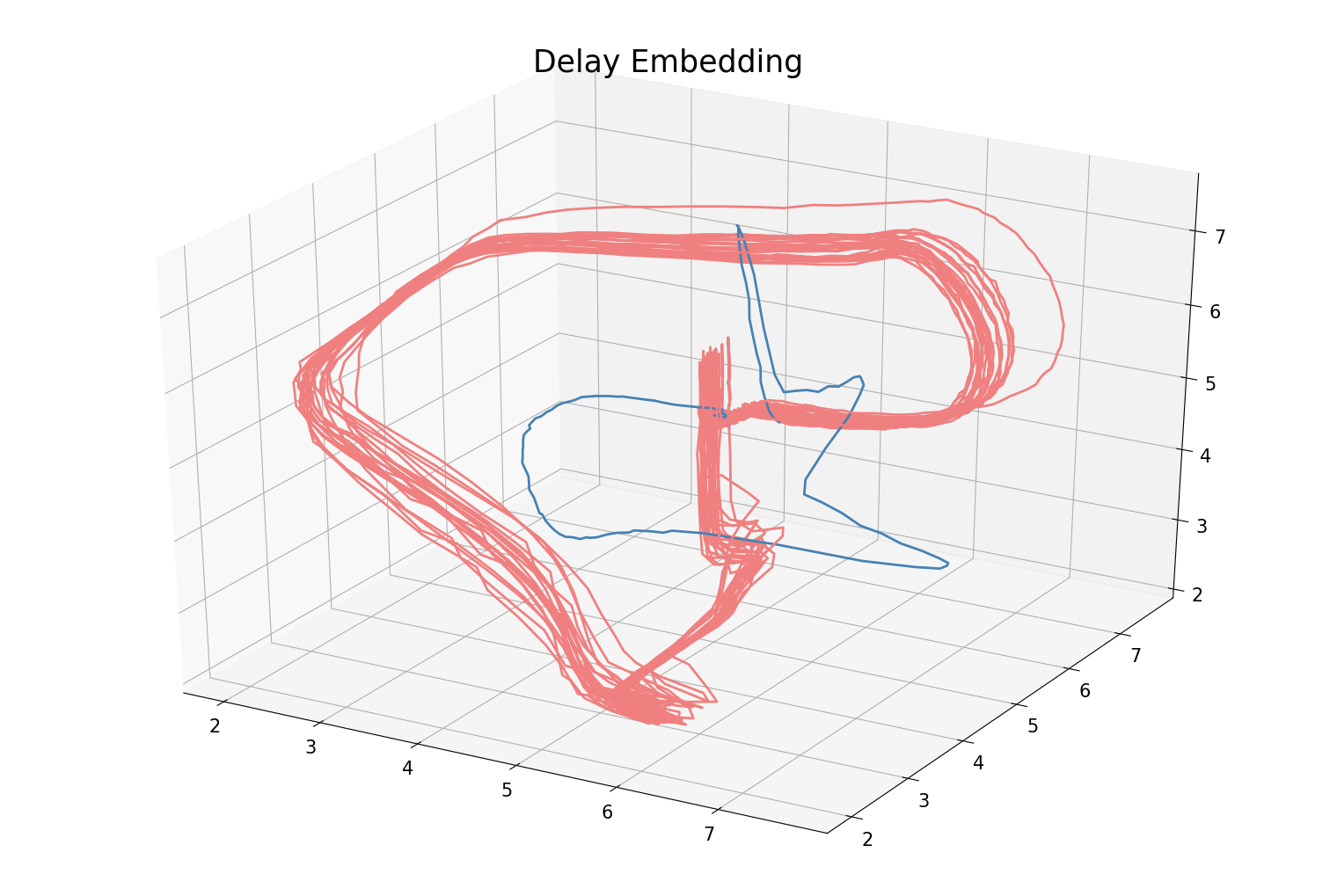}
\end{minipage}
\begin{minipage}{0.49\textwidth}
\includegraphics[width=0.49\textwidth]{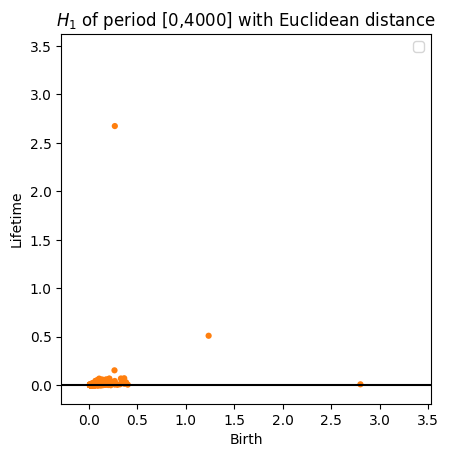}
\includegraphics[width=0.49\textwidth]{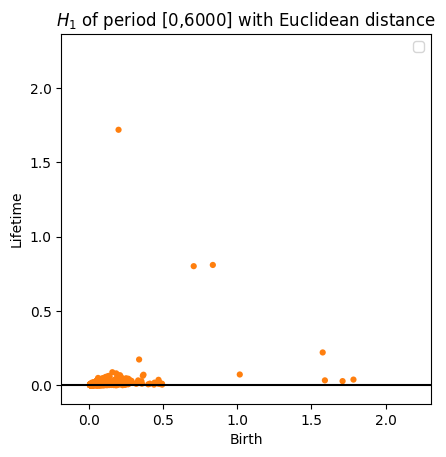}

\includegraphics[width=0.49\textwidth]{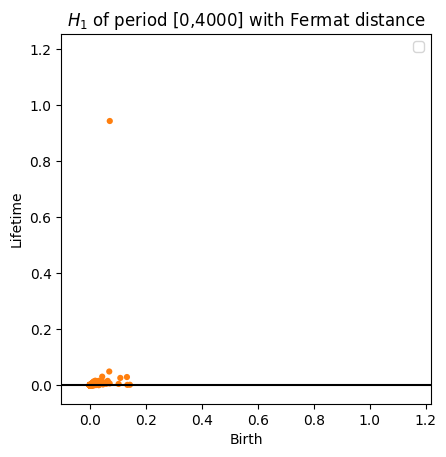}
\includegraphics[width=0.49\textwidth]{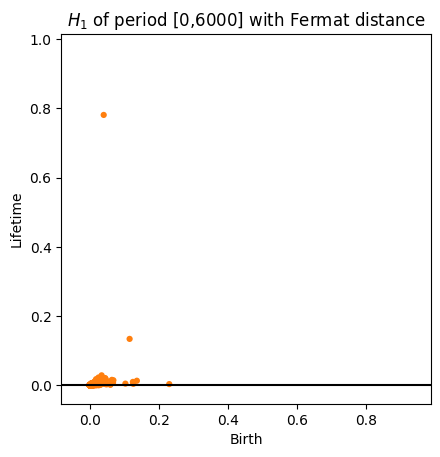}
\end{minipage}
 
\caption{Top: ECG signal (anomaly in blue). Middle: Bottleneck distance between consecutive persistence diagrams associated to time evolving embeddings of the ECG signal. Bottom: Delay embedding in $\RR^3$ with $\tau=15$. The associated persistence diagrams at degree 1 using Euclidean distance and Fermat distance with $p=2$ for the embedding of the signal in the periods of time $[0,4000]$ and $[0,6000]$.}
    \label{fig:ECG}
    \end{figure}

\par Regular heartbeats are characterized by a periodic pattern \cite[Ch.4]{lilly2016pathophysiology}.
The delay embedding in $\RR^3$ of a normal ECG has hence a cyclic topology induced by the periodic behavior of the time series  \cite[see][]{P, 6737251}. However, every time that an irregular heartbeat occurs, a new cycle arises in the embedding. We compute the associated persistence diagram for a normal period and for a period that includes an anomalous heartbeat. All delay embeddings were computed with a stride of  $t=2$, obtaining point clouds of up to $\sim 3000$ points from the original sample of size $6000$. Persistent cycles in $H_1$ in diagrams computed using Euclidean distance are not in correspondence with the periodicity pattern and the anomaly. Indeed, at the periodic interval $[0,4000]$ there are two salient generators for $H_1$. On the contrary, by using Fermat distance, an initial cycle for the periodic pattern and a second cycle in the irregular period that accounts for the anomaly are distinctly detected (here, the choice of $p=2$ is related to the weight we give to the density when computing Fermat distances; that is, we set $p$ so that the exponent $\frac{p-1}{d}$ equals $1$, where $d=1$ is the dimension of the curve).
Moreover, the moment immediately following the occurrence of the anomaly can be detected  using persistent homology of time evolving delay embeddings. Indeed, the estimator \eqref{first derivative} of the first derivative of the time evolving persistent diagrams features a prominent peak when the topology of the embedding changes. Lower peaks are also present as the result of the noisy real record.
\end{example}

\begin{example}[Pattern recognition in birdsongs] During song production, canaries use a set of air sac pressure gestures with characteristic shapes to generate different patterns of sound (or syllables). Pressure patterns of different syllables constitute a diverse set: they can be either almost harmonic oscillations, high frequency fluctuations or oscillations presenting wiggles. The recognition of song syllables from the air sac pressure series is a well-studied problem in non-linear dynamical systems \citep{mindlin2006physics, alonso2009low}.
    
We provide a topological method to detect the number of different syllables in a canary song from the (noisy) record of the fluctuations of its air sac pressure $X(t)$,  Figure \ref{fig:canary} (data provided by the Laboratory of Dynamical Systems from the Department of Physics of the University of Buenos Aires).
Given the time delay embedding of the  time series $X(t)$ with $\tau = 500$ and $D=3$, its associated persistence diagram  computed using Fermat distance with $p=1.5$ shows four prominent generators for the first homology group, which are in correspondence with the four different patterns observed in the time series (see Figure \ref{fig:canary2}). Indeed, the embedding of each syllable is topologically a cycle \cite[see][]{P, PH}. However, this decomposition is not available beforehand so the study of the global topology of the embedding of the entire time series is necessary in order to analyze the complete song. Here, prior to the computation of the persistence diagram, we down-sampled the original time series at evenly spaced times with stride  $t=100$, obtaining a subsample of size $\sim 3000$ from the original $T \sim 300000$ points.

\begin{figure}[ptb!]
\centering
\includegraphics[width=0.99\textwidth]{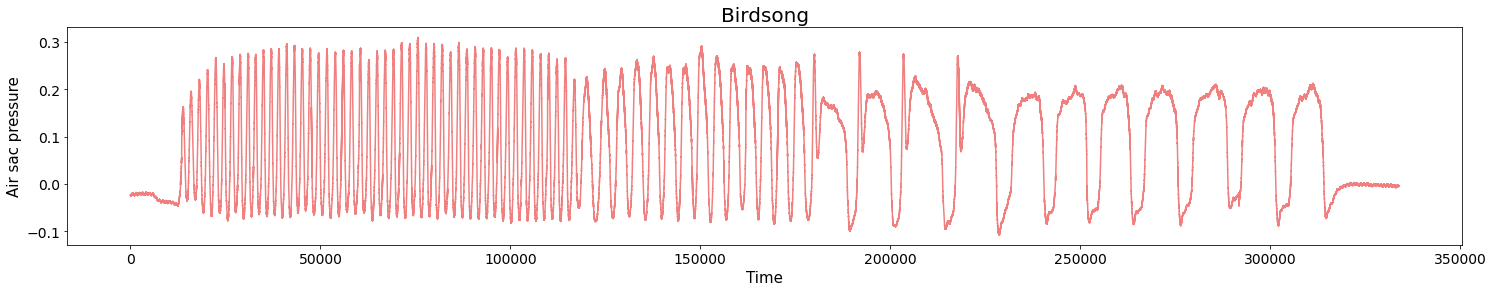}

\includegraphics[width=0.33\textwidth]{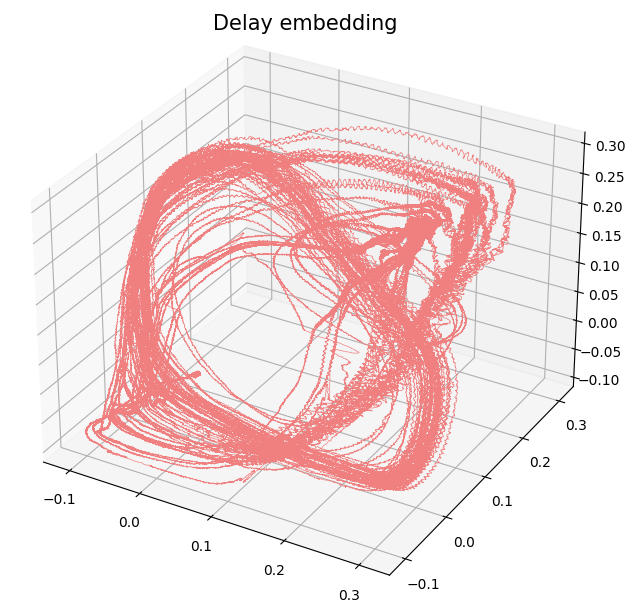}\hspace{5pt}
\hspace{30pt}
\includegraphics[width=0.29\textwidth]{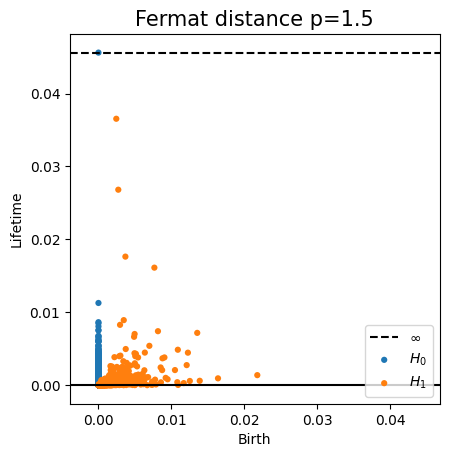}
\caption{Top:  Record of the air sac pressure of canary during a song. Bottom: Delay embedding in $\RR^3$ with time delay $\tau=500$ and its associated persistence diagram using Fermat distance with $p=1.5$.}
\label{fig:canary}
\end{figure}

We can also detect the moments at which changes of syllables take place during the song. 
The estimator \eqref{first derivative} of the first derivative of the path of persistence diagrams associated to the time evolving delay embeddings presents peaks followed by an exponential decay each time a new pattern arises, Figure \ref{fig:canary2}. 

\begin{figure}[ptb!]
\centering
\includegraphics[width=0.98\textwidth]{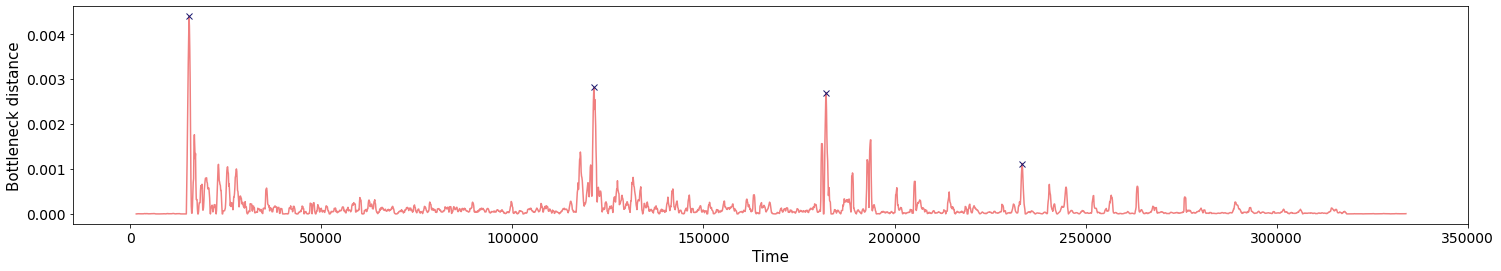}
    
\vspace{3pt}
    
\includegraphics[width=0.98\textwidth]{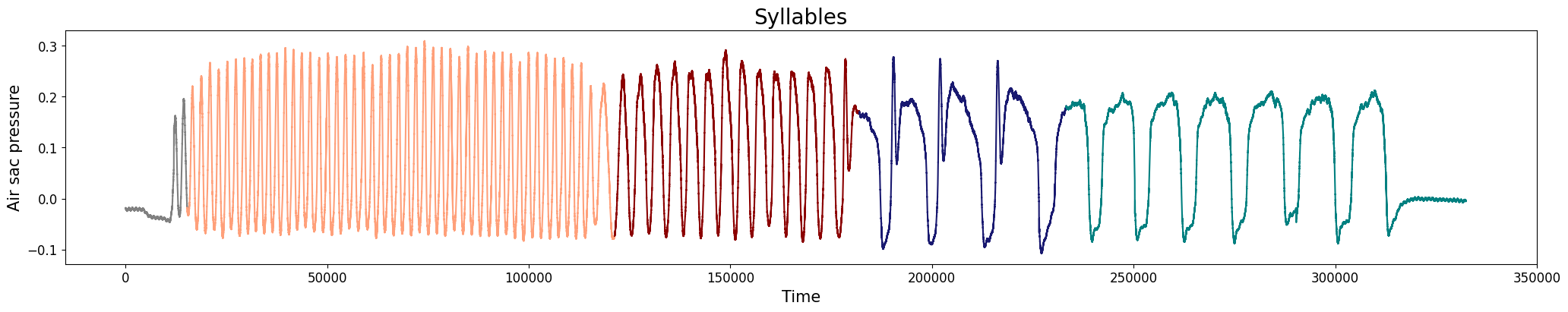}
\includegraphics[width=0.98\textwidth]{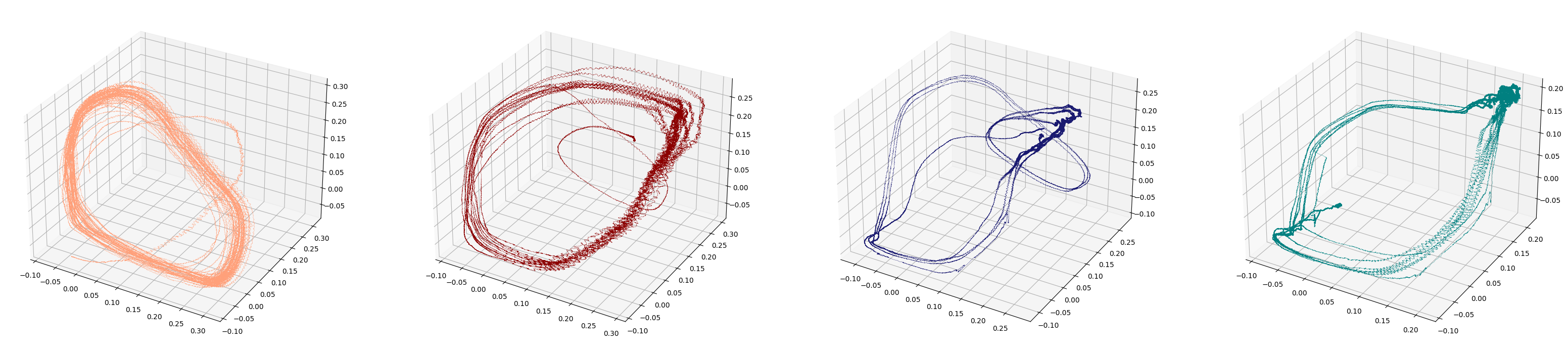}
\caption{Top: Bottleneck distance between consecutive persistence diagrams associated to time evolving embeddings (moving average curve with window of time 500). Peaks are related to changes in the pattern of the air sac pressure record of the canary song. Bottom: Delay embedding of each detected syllable.}
\label{fig:canary2}
\end{figure}
\end{example}

\section{Conclusions and Future Work} 

We introduced the use of density-based asymptotically intrinsic distances in point clouds to reconstruct the homology of a manifold from a noisy sample.
In most of the standard approaches, persistent homology computed from Euclidean samples of manifolds lacks of two relevant properties: robustness to outliers and independence of the embedding in the ambient space. Whereas each of these  properties has been studied separately in previous works, we present a simple method that is able to achieve both at the same time.

Our proposal is based on the use of Fermat distance when computing persistence diagrams of samples of manifolds. The key point is that, although this distance deforms the inherited geometry of the manifold, it produces intrinsic persistence diagrams that are more robust to outliers. Concretely, we provided rigorous proofs of convergence of the persistence diagrams of the associated metric spaces, robustness to a simple model of outliers and dependence of the persistence intervals on intrinsic (but not extrinsic) attributes of the underlying manifold.
Furthermore, we showed experimentally that our technique is stable under to a wider range of noisy situations, including real datasets. We intend to extend our results to more general models of outliers and noise in future works.
Finally, a detailed comparison of our approach with other related methods, like DTM-filtrations and the use of Euclidean distance and the intrinsic $k$-NN distance in the construction of Vietoris-Rips filtrations, is also presented.

\subsection*{Acknowledgements}
We are grateful to Luis Scoccola and Jeffrey Giansiracusa 
for many useful discussions and suggestions during the preparation of this article. We also acknowledge the anonymous reviewers and the associate editor for many helpful comments that greatly improved the manuscript. X. F. is a member of the Centre for Topological Data Analysis funded by the EPSRC grant EP/R018472/1. P. G. is partially supported by CONICET grant PIP 2021 11220200102825CO and UBACyT grant 20020190100293BA. G. M. is partially supported by PICT MAX PLANCK 4681 and PICT 00619.

\appendix
\section{Proof of Auxiliary Results}\label{proofs}

The purpose of this appendix is to present formal proofs of Proposition \ref{thm resta} and Lemma \ref{spacing}. 
Recall that $\M\subseteq \RR^D$ is a closed submanifold of dimension $d \leq D$ and $\XX_n \subseteq \M$ is an i.i.d. sample of size $n$ with common density $f>0$. Given $p>1$, we set $\alpha = 1/(d+2p)$. 
 
Proposition \ref{thm resta} will be derived from Theorem \ref{thm damenlin} \cite{HDH}.
We start with a series of results to show that any segment that is part of any shortest path with respect to $d_{\XX_n,p}$ is arbitrarily small with high probability for $n$ large enough.
This will allow us to prove that the sample Fermat distance uniformly well-approximates the power-weighted distance \eqref{weighted shortest path}.

\begin{proposition}\label{upper bound} Given  $b >0$ and $\varepsilon >0$, there exists $\theta>0$ such that 
\[\PP\left(\sup_{x,y}\left( \frac{n^{(p-1)/d} d_{\XX_n, p}(x,y)}{d_{f,p}(x,y )} - \mu \right) > \varepsilon\right) \leq \exp(-\theta n^\alpha)\]
for $n$ large enough, where the supremum is taken over all $x,y\in \M$ with $d_{\M}(x,y)\geq b$.

\end{proposition}

\begin{proof} Given $\varepsilon>0$ and $b>0$, by Theorem \ref{thm damenlin} there exists $\theta>0$ such that for every $x,y\in\M$ with $d_{\M}(x,y)\geq b$,

\begin{equation}\label{equation damelin up}
\dfrac{n^{(p-1)/d}L_{\XX_n, p}(x,y)}{d_{f, p}(x,y)}-\mu > \varepsilon
\end{equation}
with probability at most $\exp(-\theta n^{\alpha})$ (notice that here we set the sequence $b_n$ to be constantly $b$).

Let $x,y\in \M$ and let $ \gamma =  (x_0,\dots,x_{k+1})$ be the shortest path between $x,y$ with respect to $L_{\XX_n, p}$. 
That is,
\[
    L_{\XX_n, p}(x,y) = \sum_{i=0}^{k}d_{\M}(x_{i+1},x_i)^{p}.
\]
Since $|x_{i+1}-x_i|\leq d_{\M}(x_{i+1}, x_i)$,
\[
    L_{\XX_n, p}(x,y) \geq \sum_{i=0}^{k}|x_{i+1}-x_i|^{p} \geq  d_{\XX_n, p}(x,y).
\]
Thus, by \eqref{equation damelin up}, the inequality
\[
\dfrac{ n^{(p-1)/d} d_{\XX_n, p}(x,y)}{d_{f, p}(x,y)}-\mu > \varepsilon
\]
holds with probability bounded by $\exp(-\theta n^{\alpha})$.
\end{proof}

\begin{corollary}\label{global_bound}
Let $b_0>0$. 
Let $x,y\in \M$ be such that they belong to some minimal path between points in $\M$ with respect to $d_{\XX_n, p}$. Then,
\[
\PP(|x-y|>b_0)\leq \exp(-\theta n^{\alpha})
\]
for some constant $\theta > 0$, provided $n$ is large enough.
\end{corollary}

\begin{proof}
Fix $\varepsilon_0>0$.
By Proposition \ref{upper bound}, there exists a constant $\theta > 0$ such that
\[
\PP\left(\sup_{u,v} \dfrac{n^{(p-1)/d} d_{\XX_n, p}(u,v)}{ d_{f, p}(u,v)} >  \mu + \varepsilon_0 \right) \leq \exp(-\theta n^{\alpha}) 
\]
for all $n$ sufficiently large, where the supremum is taken over $u,v \in \M$ such that $d_{\M}(u,v)  \geq b_0$.
\par On the other hand, note that since $\M$ is compact the diameter $\diam_p(\M)$ of $\M$ with respect to the distance $d_{f,p}$ is finite. Hence,
\[
\dfrac{d_{f, p}(u,v)}{n^{(p-1)/d}} (\mu+\varepsilon_0)\leq \dfrac{\diam_p(\M)}{n^{(p-1)/d}} (\mu+\varepsilon_0) \leq b_0^{p}
\]
for all $u,v \in \M$ with $d_{\M}(u,v) \geq b_0$ and all $n$ sufficiently large.
\par Suppose now that $x,y\in \M$ belong to some shortest path between points of $\M$ with respect to $d_{\XX_n, p}$, say $u$ and $v$, but that $|x-y| > b_0$. Then, clearly $d_{\XX_n, p}(u,v) \geq |x-y|^p$ and $d_{\M}(u,v) > b_0$ (since otherwise $d_{\XX_n,p}(u,v) \leq |u-v|^{p} < b_0^{p}$). We remark here that $x$ and $y$ do not necessarily belong to the sample $\XX_n$. From the previous computations, it follows that whenever $n$ is large enough, with probability at least $1 - \exp(-\theta n^{\alpha})$,
\[
 |x-y|^{p} \leq d_{\XX_n, p}(u,v) \leq \dfrac{d_{f,p}(u,v)}{n^{(p-1)/d}} (\mu+\varepsilon_0) \leq b_0^{p}, \]
as we wanted to show.
\end{proof}

\begin{remark}\label{boiss} (see \citealp[Corollary 4]{isomapproofs} or \citealp[Lemma 3]{BLW})
Let $(\M, g)$ be a smooth compact Riemannian manifold embedded in $\RR^D$. Given $\delta >0$, there exists $\varepsilon>0$ such that for every $x,y\in \M$ with $|x-y|<\varepsilon$, 
\[
    d_{\M}(x,y)\leq (1+\delta)|x-y|.
\]
\end{remark}

We are now able to prove a new version of Theorem \ref{thm damenlin} in which the proposed estimator of $d_{f,p}$ is the sample Fermat distance (rather than the power-weighted shortest path). 

\begin{proposition}
\label{thm quotient}
Fix $\varepsilon > 0$ and a sequence of positive real numbers $(b_n)_{n \geq 1}$ satisfying that $\frac{\log(n)}{n b_n^{d}} \to 0$ when $n \to \infty$. Then, for every $p>1$, there exists $\theta>0$ such that
\[
\PP\left(\sup_{x,y}\left| \frac{n^{(p-1)/d} d_{\XX_n, p}(x,y)}{d_{f,p}(x,y )} - \mu \right|>\varepsilon\right) \leq \exp\left(-\theta (n b_n^{d})^\alpha\right)
\]
for $n$ large enough, where the supremum is taken over $x,y\in \M$ with $d_{\M}(x,y)\geq b_n$.
\end{proposition}

\begin{proof} Let $\delta > 0$ be a small number to be fixed later. The strategy of the proof consists of showing that, with probability exponentially high in $(nb_n^d)^{\alpha}$, $L_{\XX_n,p}(x,y)$ and $d_{\XX_n,p}(x,y)$ coincide up to a factor of $(1+\delta)^{p}$ for all $x,y \in \M$ with $d_{\M}(x,y) \geq b_n$. Once that is established, the proof follows readily by applying Theorem \ref{thm damenlin}.
\par Notice in first place that by Remark \ref{boiss}, there exists $\eta > 0$ such that $d_{\M}(x,y) \leq (1+\delta)|x-y|$ whenever $x,y \in \M$, $|x-y| < \eta$. By Corollary \ref{global_bound}, we may assume that $|u-v| < \eta$ for every $u, v \in \M$ belonging to a minimal path with probability exponentially high in $n^{\alpha}$. Let $x,y \in \M$ be two points with $d_{\M}(x,y) \geq b_n$. Since by our assumptions every segment in a shortest path from $x$ to $y$ with respect to $d_{\XX_n,p}$ has Euclidean length at most $\eta$, it is not difficult to see that
\begin{equation}\label{eq:bilipschitz}
    d_{\XX_n,p}(x,y) \leq L_{\XX_n,p}(x,y) \leq (1+\delta)^{p}d_{\XX_n,p}(x,y).
\end{equation}

Now, by Theorem \ref{thm damenlin}, the probability that

\begin{equation}\label{eq:damelin_ineq}
    \left| \frac{n^{(p-1)/d} L_{\XX_n, p}(x,y)}{d_{f,p}(x,y )} - \mu \right|<\frac{\varepsilon}{2}
\end{equation}
is exponentially high in $(n b_n^{d})^{\alpha}$, provided $n$ is large enough. We will check that for $\delta > 0$ sufficiently small, the desired inequality for $d_{\XX_n,p}$ follows if we assume that the event from \eqref{eq:damelin_ineq} occurs. It is clear by \eqref{eq:bilipschitz} and \eqref{eq:damelin_ineq} that
\[
    \frac{n^{(p-1)/d} d_{\XX_n, p}(x,y)}{d_{f,p}(x,y )} - \mu < \frac{\varepsilon}{2}.
\]
As for the other inequality, notice that
\[
    -\frac{\varepsilon}{2} < (1+\delta)^{p} \left(\frac{n^{(p-1)/d} d_{\XX_n, p}(x,y)}{d_{f,p}(x,y )} - \mu\right) + ((1+\delta)^{p} - 1) \mu.
\]
Hence, for $\delta > 0$ small enough we have
\[
    -\varepsilon < \frac{n^{(p-1)/d} d_{\XX_n, p}(x,y)}{d_{f,p}(x,y )} - \mu
\]
as desired.
\end{proof}

Finally, we promote the convergence of the sample Fermat distance from Proposition \ref{thm quotient} to a \textit{uniform} convergence in probability (that is, for any pair of points $x,y \in\M$ regardless of the distance between them).
Such uniform convergence may be accomplished by choosing a sequence $(b_n)_{n \geq 1}$ which converges to $0$ at an adequate rate.
This step is instrumental in order to prove the Gromov--Hausdorff convergence of the sample metric spaces $(\XX_n, d_{n,p})$ to $(\M, d_{f,p})$ (see Theorem \ref{conv.diag} and its proof).

\begin{proof} [Proposition \ref{thm resta}] 
Roughly, the strategy of the proof consists in bounding the quantity
\[
    |n^{(p-1)/d} d_{\XX_n,p}(x,y) - \mu d_{f,p}(x,y)|
\]
splitting in two cases according to whether the distance $d_{\M}(x,y)$ is greater than or smaller than some appropriately chosen sequence $b_n > 0$. More precisely, we will set $b_n = n^{-\lambda}$ for some $\lambda \in ((p-1)/pd, 1/d)$. Let $\varepsilon > 0$. Since $\lambda < 1/d$, clearly the sequence $\left(\frac{\log(n)}{n b_n^d}\right)_{n \geq 1}$ converges to $0$ as $n$ goes to infinity and hence, by Proposition \ref{thm quotient} the bound
\[
    \left| \frac{n^{(p-1)/d} d_{\XX_n, p}(x,y)}{d_{f,p}(x,y )} - \mu \right|>\varepsilon'
\]
holds with probability at most $\exp(-\theta(n b_n^{d})^{\alpha}) = \exp(-\theta n^{(1-\lambda d) \alpha})$ for some $\theta > 0$ and all $x,y \in \M$ with $d_{\M}(x,y) \geq n^{-\lambda}$ provided $n$ is large enough (here $\varepsilon' > 0$ is a small number to be determined). Denote by $\diam(\M)$ the diameter of $\M$ with respect to the distance $d_{\M}$. Since $d_{f,p}(x,y) \leq m_f^{-(p-1)/d} d_{\M}(x,y) \leq m_f^{-(p-1)/d} \diam(\M)$, we see that the event
\[
    |n^{(p-1)/d} d_{\XX_n, p}(x,y) - \mu d_{f,p}(x,y)| > m_f^{-(p-1)/d} \diam(\M) \varepsilon'
\]
also holds with probability bounded from above by $\exp(-\theta n^{(1-\lambda d) \alpha})$ for the same $\theta > 0$ as before, whenever $d_{\M}(x,y) \geq n^{-\lambda}$. By setting $\varepsilon' = \varepsilon (m_f^{-(p-1)/d} \diam(\M))^{-1}$ we obtain the desired bound for $x,y \in \M$ with $d_{\M}(x,y) \geq n^{-\lambda}$. For the remaining case, take $x,y \in \M$ satisfying $d_{\M}(x,y) \leq n^{-\lambda}$ and notice in first place that
\[
    d_{f,p}(x,y) \leq m_f^{-(p-1)/d} d_{\M}(x,y) \leq m_f^{-(p-1)/d} n^{-\lambda}.
\]
Hence, for $n$ sufficiently large, $\mu d_{f,p}(x,y) \leq \varepsilon/2$. On the other hand, since by definition of $d_{\XX_n,p}$ it is
\[
    d_{\XX_n,p}(x,y) \leq |x-y|^{p} \leq d_{\M}(x,y)^{p} \leq n^{-\lambda p},
\]
we see that $n^{(p-1)/d} d_{\XX_n,p}(x,y) \leq n^{(p-1)/d - \lambda p}$. The hypothesis on $\lambda$ implies that the exponent of $n$ in the last inequality is negative and thus $n^{(p-1)/d}d_{\XX_n,p}(x,y) \leq \varepsilon/2$ provided $n$ is large. Summing up, we conclude that there exists $n_0$ such that for all $x,y \in \M$ with $d_{\M}(x,y) \leq n^{-\lambda}$ and $n \geq n_0$,
\[
    |n^{(p-1)/d} d_{\XX_n,p}(x,y) - \mu d_{f,p}(x,y)| \leq \varepsilon,
\]
which completes the proof of the proposition.
\end{proof}

We turn now to the proof of Lemma \ref{spacing}, which follows ideas from \cite{CRC} and \cite[Section 5]{MR2147326}.

\begin{definition}\cite[see][Chapter 5]{Lee}
The \textit{injectivity radius} $\inj(\N)$ of a Riemannian manifold  $(\N,g)$ is defined as
\[
\inj(\N) := \inf_{x \in N} \inj(\N,x),
\]
where $\inj(\N,x)$ is the largest radius for which the exponential map is a diffeomorphism.
\end{definition}

\begin{proof} [Lemma \ref{spacing}] Since $\M$ is compact, its injectivity radius $\inj(\M)$ is strictly positive. Then, by an inequality of Croke \cite[see][Proposition 14]{Croke}, there exists a constant $c = c(d) > 0$ such that every metric ball $B$ in $\M$ of radius $r < \frac{\inj(\M)}{2}$ has volume at least $c(d) r^{d}$. Since we can assume that $\kappa < 1$ without loss of generality, for all $n$ sufficiently large we have $n^{(\kappa-1)/d} < \frac{\inj(\M)}{2}$. From this point, we follow the strategy from the proof of \cite[Theorem 3]{CRC}. Let $P_n$ be the maximum number of disjoint balls of radius $\frac{n^{(\kappa-1)/d}}{4}$ contained in $\M$ --- this is known as \textit{packing number}, see for example \cite[Section 5]{NSW} --- and take $\{ B_1, \dots, B_{P_n}\}$ a set of disjoint balls of radius $\frac{n^{(\kappa-1)/d}}{4}$ in $\M$. It is clear then that
\[
    P_n \leq \frac{\Vol(\M)}{\min_{1 \leq j \leq P_n} \Vol(B_j)} \leq \frac{\Vol(\M)4^{d}}{c(d)}n^{1-\kappa},
\]
for $n$ so large that $n^{(\kappa-1)/d} < \frac{\inj(\M)}{2}$.
Now, suppose that $x \in \M$ verifies $d_{\M}(x, \XX_n) > n^{(\kappa-1)/d}$. Since the balls $2B_1, \dots, 2B_{P_n}$ cover $\M$ (where $2B_j$ stands for the ball with the same center as $B_j$ but with twice the radius) the distance from $x$ to some center of these balls is at most $\frac{n^{(\kappa-1)/d}}{2}$ and thus there should be no point from the sample in some ball $2B_j$.
A simple computation reveals that the probability that some random variable $\mathbf{x}_i \in \XX_n$ does not belong to $2B_j$ is at most $1 - m_f\cdot \Vol(2B_j)$. By the independence of the random variables $\mathbf{x}_1,\dots, \mathbf{x}_n $, if $n$ is large enough
\[ \PP\left(\bigcap_{i=1}^{n} \{\mathbf{x}_i \not\in 2B_j\} \right) \leq \big(1 - m_f\cdot \Vol(2B_j)\big)^{n} \leq \big(1 - m_f c(d) n^{\kappa-1}\big)^n.
\]
We conclude that
\[
    \PP\left(\left\{ \sup_{x \in \M} d_{\M}(x,\XX_n) \geq n^{(\kappa-1)/d} \right\}\right)  \leq \sum_{j=1}^{P_n}  \PP\left(\bigcap_{i=1}^{n} \{\mathbf{x}_i \not\in 2B_j\}\right) \leq (1 - m_f c(d) n^{\kappa-1})^n P_n .
\]
Since $P_n$ grows at most like a polynomial in $n$, $ (1 - m_f c(d) n^{\kappa-1})^nP_n  \le \exp(-\theta n^{\kappa})$
for an appropriate $\theta >0$ and $n$ big enough, as we wanted to show.
\end{proof}

\bibliographystyle{plain}
\bibliography{biblio}
\end{document}